%% file: DLM-optimal.tex
\documentclass[english]{article}
\usepackage[T1]{fontenc}
\usepackage[latin9]{inputenc}
\usepackage{geometry}
\usepackage[unicode=true,
 bookmarks=false,
 breaklinks=false,
 pdfborder={0 0 1},
 colorlinks=false]{hyperref}
\hypersetup{colorlinks,
linkcolor=red,
anchorcolor=blue,
citecolor=blue,
urlcolor=blue}

\usepackage{bm}
\usepackage{amsmath,amsthm,amssymb,amsfonts}
\usepackage{mathtools}
\usepackage{comment}
\usepackage{natbib}
\usepackage{enumerate}
\usepackage{graphicx}
\usepackage{algorithm}
\usepackage{algpseudocode}
\usepackage{float}
\usepackage{multirow}
\usepackage{dsfont}
\usepackage{tcolorbox}
\usepackage{tabulary}
\usepackage{colortbl}
\usepackage{color}
\usepackage{makecell}
\usepackage{subfigure}

\definecolor{cc}{RGB}{231,117,0}
\definecolor{yx}{RGB}{0, 200, 0}

\allowdisplaybreaks

\newcommand{\defn}{\coloneqq}
\newcommand{\defnrev}{\eqqcolon}

\newcommand{\cS}{\mathcal{S}}
\newcommand{\cA}{\mathcal{A}}
\newcommand{\cB}{\mathcal{B}}
\newcommand{\cC}{\mathcal{C}}

\newcommand{\cI}{\mathcal{I}}
\newcommand{\cJ}{\mathcal{J}}

\newcommand{\cT}{\mathcal{T}}
\newcommand{\cP}{\mathcal{P}}

\newcommand{\bE}{\mathbb{E}}
\newcommand{\bP}{\mathbb{P}}
\newcommand{\bR}{\mathbb{R}}

\newcommand{\bX}{\mathbb{X}}

\newcommand{\KL}{\mathsf{KL}}
\newcommand{\diff}{\,\mathrm{d}}

\newcommand{\numpf}[2]{\overset{(\mathrm{#1})}{#2}}
\newcommand{\ol}{\overline}
\newcommand{\wh}{\widehat}
\newcommand{\wt}{\widetilde}

\newcommand{\data}{\mathsf{data}}

\newcommand{\veps}{\varepsilon}

\newcommand{\setc}{\mathrm{c}}

\newcommand{\E}{\mathbb{E}}
\newcommand{\TC}{\mathsf{TC}}
\newcommand{\DTC}{\mathsf{DTC}}
\newcommand{\tc}{\mathsf{tc}}
\newcommand{\dtc}{\mathsf{dtc}}
\newcommand{\mask}{\mathsf{M}}
\newcommand{\train}{\mathsf{pred}}
\newcommand{\out}{\mathsf{out}}

\title{Adaptation to Intrinsic Dependence in Diffusion Language Models}

\author{
  Yunxiao Zhao\thanks{Department of Industrial and Operations Engineering, University of Michigan, Ann Arbor, USA.
  {Email: \texttt{\{zhyx,cxcai\}@umich.edu}.}
  }
  \and
  Changxiao Cai\footnotemark[1]
}
\date{\today}

\begin{document}

\theoremstyle{plain} 
\newtheorem{lemma}{\bf Lemma} 
\newtheorem{proposition}{\bf Proposition}
\newtheorem{theorem}{\bf Theorem}
\newtheorem{corollary}{\bf Corollary} 
\newtheorem{claim}{\bf Claim}

\theoremstyle{remark}
\newtheorem{assumption}{\bf Assumption} 
\newtheorem{definition}{\bf Definition} 
\newtheorem{condition}{\bf Condition}
\newtheorem{property}{\bf Property} 
\newtheorem{example}{\bf Example}
\newtheorem{fact}{\bf Fact}
\newtheorem{remark}{\bf Remark}

\maketitle 
\setcounter{tocdepth}{2}

\input{abstract}

\medskip

\tableofcontents{}

\input{intro}

\input{related-work}
\input{notation}
\input{problem}

\input{results}
\input{analysis.tex}

\input{experiment.tex}
\input{discussion.tex}

\section*{Acknowledgements}

C.\ Cai is supported in part by the NSF grants DMS-2515333.

\bibliographystyle{apalike}
\bibliography{bibfileDF,bibfileLLM}

\appendix
\input{proof-TC.tex}
\input{proof-TC-f.tex}
\input{proof-DTC.tex}

\input{proof-DTC-g.tex}
\input{proof-aux.tex}

\end{document}

%% file: abstract.tex
\begin{abstract}

Diffusion language models (DLMs) have recently emerged as a promising alternative to autoregressive (AR) approaches, enabling parallel token generation beyond a rigid left-to-right order.
Despite growing empirical success, the theoretical understanding of how unmasking schedules---which specify the order and size of unmasked tokens during sampling---affect generation quality remains limited.
In this work, we introduce a distribution-agnostic unmasking schedule for DLMs that adapts to the (unknown) dependence structure of the target data distribution, without requiring any prior knowledge or hyperparameter tuning.
In contrast to prior deterministic procedures that fix unmasking sizes, our method randomizes the number of tokens revealed at each iteration.
We show that, for two specific parameter choices, the sampling convergence guarantees---measured by Kullback-Leibler (KL) divergence---scale as $\widetilde O(\mathsf{TC}/K)$ and $\widetilde O(\mathsf{DTC}/K)$ respectively. Here, $K$ is the number of iterations, and $\mathsf{TC}$ and $\mathsf{DTC}$ are the total correlation and dual total correlation of the target distribution, capturing the intrinsic dependence structure underlying the data.
Importantly, our guarantees hold in the practically relevant parallel-sampling regime $K<L$ where $L$ is the token sequence length.
These results significantly improve upon prior convergence theories and yield substantial sampling acceleration for low-complexity distributions. 
Overall, our findings unveil the adaptivity of DLMs to intrinsic data structures and shed light on the benefit of randomized unmasking sizes in inference schedule design.
\end{abstract}

%% file: intro.tex
\section{Introduction}
\label{sec:intro}

Diffusion models \citep{sohl2015deep,ho2020denoising,song2020denoising} have become the cornerstone of modern generative modeling in continuous domains such as images and audio~\citep{dhariwal2021diffusion,rombach2022high}.
%
Motivated by these successes, recent work has extended diffusion-based generation to discrete data, most notably natural language, giving rise to \textit{diffusion language models} (DLMs) \citep{sahoo2024simple,nie2025large}.
Unlike autoregressive (AR) paradigms, which decode tokens sequentially in a fixed left-to-right order, DLMs allow multiple tokens to be generated in parallel and in arbitrary orders, offering a promising route toward faster and more flexible inference in large language models (LLMs).

State-of-the-art DLMs are built on \textit{masked diffusion models} (MDMs) \citep{austin2021structured,campbell2022continuous,shi2024simplified}.
An MDM introduces an absorbing state called \textit{mask} and consists of two complementary processes. The forward (masking) process transforms a token sequence $X^{(0)}$ drawn from a data distribution $p_\data$ into a fully masked sequence $X^{(K)}$ by gradually masking tokens:
\begin{align*}
    X^{(0)} \,\overset{\mathsf{mask}}{\longrightarrow}\, X^{(1)}
        \,\overset{\mathsf{mask}}{\longrightarrow}\, X^{(2)}
     \,\overset{\mathsf{mask}}{\longrightarrow}\, \cdots
     \,\overset{\mathsf{mask}}{\longrightarrow}\, X^{(K)}.
\end{align*}
The reverse (unmasking) process seeks to reverse the forward process and iteratively converts a fully masked sequence $Y^{(0)}$ into a fresh sample $Y^{(K)}$ from the data distribution:
\begin{align*}
    Y^{(0)} \,\overset{\mathsf{unmask}}{\longrightarrow}\, Y^{(1)}
     \,\overset{\mathsf{unmask}}{\longrightarrow}\, Y^{(2)} 
     \,\overset{\mathsf{unmask}}{\longrightarrow}\, \cdots
     \,\overset{\mathsf{unmask}}{\longrightarrow}\, Y^{(K)}.
\end{align*}
The core mechanism of the reverse process is to choose a subset of indices to unmask at each iteration and sample their values conditioned on the currently revealed context.
In principle, the idealized update would draw from the true conditional joint distribution of the tokens to be unmasked given the already unmasked tokens. 
Hence, DLM sampling relies on a \textit{mask predictor} that learns these conditional distributions from finite training samples drawn from the target distribution.
One can view the mask predictor as a counterpart of the score estimator in continuous diffusion models~\citep{song2019generative} that enable the transformation from noise to data.

However, exact reverse sampling is computationally prohibitive, since the joint distribution of multiple tokens is generally intractable in high-dimensional settings. 
To address this issue, practical DLMs typically adopt a conditional independence approximation \citep{nie2025large,wu2025fast}, using the \textit{product of conditional marginal distributions} to approximate the true conditional joint distribution of masked tokens.
%
In other words, the mask predictor is trained to learn the conditional marginal distributions of individual unmasked tokens, and predict each masked token independently given the unmasked context.


\subsection{Trade-off between sampling efficiency and accuracy}


A key appeal of DLMs is their potential to break the sequential bottleneck of AR generation. A neural network-based mask predictor can simultaneously compute the conditional marginal distributions for all masked positions in a single forward pass, making it possible to generate multiple tokens in parallel at each iteration.

This parallelism, however, comes with a fundamental statistical price. Since we approximate the conditional joint by a product of conditional marginals and sample each newly unmasked token independently given the \textit{same} partially revealed context, this approximation ignores dependencies among simultaneously generated tokens. Therefore, it induces an inevitable sampling error---even when the true conditional marginals are available.

Consequently, given a fixed pre-trained mask predictor, the sampling efficiency and accuracy of DLM sampling is governed by the \textit{unmasking schedule}, which specifies how many tokens to reveal and which positions to choose at each iteration. 
For example, unmasking all tokens in a single iteration maximizes parallelism but incurs the largest dependence-induced error, whereas unmasking one token at a time eliminates this error but reduces sampling to an effectively AR procedure.

%

Designing unmasking schedules that achieve high-quality generation without sacrificing the computational gains of parallel decoding is therefore a central challenge for DLMs.
Although a range of heuristics have been explored in practice, such as uniform random schedules~\citep{shi2024simplified,sahoo2024simple,nie2025large} and confidence-based procedures~\citep{ben2025accelerated}, a principled theoretical understanding of DLM sampling remains limited.

\paragraph{Prior theory on DLM sampling.}
Recent work has begun to theoretically investigate this speed-accuracy trade-off by analyzing the sampling error in the regime where the number of sampling iterations~$K$ is smaller than the sequence length~$L$.


The work of \citet{li2025breaking} provided the first convergence analysis for DLM sampling, focusing on the uniform random unmasking scheme. It prescribes a batch-size schedule $s_1,\dots,s_K$ with $s_k\geq 1$ and $\sum_{k=1}^K s_k = L$, and at iteration $k$, unmasks $s_k$ tokens selected uniformly at random from the remaining masked tokens.
When the unmasking sizes are balanced across iterations (i.e., $s_k \asymp L/K$ for all $k$), they show that the sampling convergence rate, measured by KL divergence, is bounded by
\begin{align*}
O \bigg( \frac1K \sum_{i=1}^L I(X_i; X_{-i}) \bigg), 
\end{align*}
where $I(X_i; X_{-i})$ is the mutual information between the $i$-token $X_i$ and the rest $X_{-i}\defn(X_j)_{j \neq i}$ under the data distribution $X\sim p_{\data}$.
Similar bounds were also established in \citet{lavenant2025error} later.

Recent work of \citet{chen2025optimal} sharpened the picture by giving an exact characterization of the sampling convergence rate.
The authors then proposed an unmasking schedule that assumes access to accurate estimates $\widehat \TC$ and $\widehat \DTC$ of the \textit{total correlation} (TC) and \textit{dual total correlation} (DTC) of the data distribution (see the formal definitions in Section~\ref{sec:problem}).
They proved that this scheme achieves the 
sampling convergence rate
\begin{align*}
O \bigg( \frac{\log L}{K} \min\{\widehat \TC, \widehat \DTC\} \bigg) .
\end{align*}

Since the TC and DTC satisfy
$
\TC(X) + \DTC(X) = \sum_{i=1}^L I(X_i; X_{-i}),
$
this result improves upon that of \citet{li2025breaking} and demonstrates that, for structured distributions where $\min\{\TC(X),\DTC(X)\}$ is much smaller than $\sum_{i} I(X_i; X_{-i})$, appropriately designed unmasking schedules can significantly improve sampling speed.

\paragraph{Gap.}
Despite these theoretical advances, several critical limitations remain.
First, the analysis in \citet{li2025breaking} did not take intrinsic data structures such as weak dependence into account. The sampling error of the uniform unmasking strategy scales with the sum of mutual information $\sum_i I(X_i;X_{-i})$, which can be suboptimal for structured data distributions.
For instance, in regimes where $I(X_i;X_{-i}) \gtrsim 1$---commonly encountered in practice---the sampling error grows linearly with $L$, making it ineffective for generating long sequences.

While the unmasking schedule proposed by \citet{chen2025optimal} can achieve improved sampling guarantees that depend on the intrinsic structure of the data distribution, it requires reliable estimates of the information-theoretic quantities such as TC, DTC, or related information curves, which are typically unavailable for real-world high-dimensional data distributions.    
In many cases, accurately estimating these quantities may require as much (or more) data than training the DLM itself. Consequently, their schedule is best viewed as an information-theoretic benchmark rather than a directly implementable procedure.

This gap motivates the question that we explore in this paper:
\begin{center}
\textit{Can we design an unmasking schedule that adapts to unknown data structure without requiring any prior knowledge?}
\end{center}

\subsection{Our contributions}
In this paper, we answer the above question affirmatively by proposing a distribution-agnostic, randomized unmasking schedule, revealing how intrinsic dependence structure in the data can be harnessed to speed up sampling.

Specifically, we develop a randomized unmasking strategy that, at each iteration, first randomly selects the unmasking set size from a carefully designed distribution and then unmasks the corresponding number of tokens uniformly at random. 
Crucially, the schedule depends only on the  sequence length~$L$ and the number of iterations~$K$, and requires no prior knowledge or estimation of the target distribution.

We prove that given a pre-trained mask predictor, the resulting sampler achieves expected KL-sampling error bounds that depend on intrinsic information-theoretic measures. With two specific parameter choices, we obtain sampling convergence guarantees
\begin{align*}
\frac{\TC(X)}{K}\log L \qquad\text{and}\qquad  \frac{\DTC(X)}{K - \log (L-1)-1}(\log (L-1)+1),
\end{align*}
where $\TC(X)$ and $\DTC(X)$ denote the TC and DTC of the data distribution, respectively.
Notably, these guarantees hold in the practically relevant parallel generation regime $K<L$.
When the target distribution has low complexity, i.e., $\min\{\DTC(X), \TC(X)\} \ll L$ --- frequently encountered in many practical applications --- our results reveal that DLMs can automatically leverage the favorable low-dimensional structures, without distribution-specific knowledge, estimation, or tuning.
More broadly, these findings deliver the first theory for the dependence-adaptivity of DLMs, and strengthen the state-of-the-art convergence theory.  
Finally, we note that the core design choice of our approach is randomizing unmasking sizes instead of fixing them in advance.
This may shed light on the unmasking schedule design for DLMs in practice.

%% file: related-work.tex

\subsection{Related work}
\label{sec:related-work}

\paragraph{Diffusion language models.}

Early attempts to bring diffusion models to language modeling operated on the continuous embeddings of discrete tokens and applied Gaussian diffusion in continuous spaces \citep{li2022diffusion}.
A major shift came with discrete diffusion models, where \citet{austin2021structured} formulates discrete diffusion as a continuous-time Markov chain (CTMC) with categorical transition kernels, providing a unified framework that encompasses uniform-noise, absorbing (mask), and other structured corruption processes.
The work of \citet{lou2023discrete} proposed score entropy discrete diffusion, introducing a score entropy objective to learn concrete scores and improving diffusion language modeling performance.
Recently, MDMs have proven particularly effective for language generation \citep{sahoo2024simple,shi2024simplified}, and scaling these approaches has led to performance competitive with autoregressive models \citep{nie2025large,labs2025mercury,song2025seed,you2025llada,zhu2025llada}. Beyond mask-based corruption, uniform-noise discrete diffusion has also been explored \citep{schiff2024simple}.

\paragraph{Unmasking schemes.}
A common practical unmasking strategy follows a two-stage template: first specifies a size schedule and then selects that many positions, often uniformly at random, from the remaining masked tokens. Popular deterministic size schedules include cosine schedules \citep{shi2024simplified} and log-linear schedules \citep{sahoo2024simple,lou2023discrete}.
A different line of heuristics is to unmask ``easy'' tokens first, using either confidence-based rules that prioritize positions with high predicted probability \citep{nie2025large,yu2025dimple,wu2025fast} or entropy-based rules that select positions with low predictive entropy \citep{ben2025accelerated}. 
However, such greedy choices can also be myopic: committing early to low-entropy tokens can increase uncertainty for the remaining positions in later iterations by removing informative context or amplifying dependency effects. To address this limitation, more structured approaches have been proposed, including learning unmasking policies via reinforcement learning \citep{jazbec2025learningunmaskingpoliciesdiffusion} and schedule design via planning or search procedures \citep{fu2025bitsroundsparalleldecoding}.

\paragraph{Theory for diffusion models.}

Our results connect to a well-developed literature on convergence guarantees for continuous diffusion samplers \citep{lee2022convergence,chen2022sampling,lee2023convergence,chen2023improved,benton2023linear,li2024d,li2024sharp}, which studies given a pre-trained model, how fast the sampling iterate converges to the target distribution in $\bR^D$ as the number of iterations $K$ increases.
State-of-the-art theory \citep{jiao2025optimal,zhang2025sublinear} establishes convergence rates of $\wt{O}(D/T^2)$ in KL divergence and $\wt{O}(\sqrt{D}/K)$ in total variation distance, and a variety of provably accelerated samplers have been developed \citep{li2024provable,li2024improved,gatmiry2026high,huang2024convergence,li2024accelerating,li2025faster}.
The adaptivity to (unknown) low-complexity structure has also been studied recently \citep{li2024adapting,liang2025low,huang2024denoising,potaptchik2024linear,li2025dimension}, improving the dependence on the ambient dimension $D$ to some intrinsic dimension $d$ of the data distribution.
Thus for $d \ll D$, continuous diffusion samplers can exploit low-dimensional structures to accelerate sampling.
In contrast, convergence theory for discrete diffusion models remains comparatively nascent. Existing analysis typically requires the number of iterations to scale at least with the data dimension or sequence length \citep{chen2024convergence,liang2025absorb,feng2025theoretical,ren2025fast,ren2024discrete}, and therefore does not accommodate the parallel-generation regime of DLMs.
Beyond convergence guarantees, end-to-end statistical guarantees by tracking the error contributions from both the training and sampling phases have also been developed \citep{oko2023diffusion,wibisono2024optimal,zhang2024minimax,cai2025minimax,gatmiry2024learning,chen2024learning}.

\paragraph{Concurrent work.}
In a concurrent and independent work, \citet{dmitriev2026efficient} propose a $\tau$-leaping-based sampler for MDMs under a CTMC formulation. Their method also adapts to intrinsic data structure, with sampling convergence rate characterized by the TC and DTC of the target data distribution.
Our emphasis is complementary: we study the explicit design of unmasking schedules directly in discrete-time DLMs, without invoking a CTMC perspective.
In addition, our iteration complexity does not scale with the vocabulary size, while the $\tau$-leaping approach incurs a dependence on the vocabulary size.
Furthermore, our theoretical analysis directly tracks the error propagation across unmasking iterations, which may be of independent interest for studying DLM samplers.

%% file: notation.tex
\subsection{Notation}
\label{sec:notation}
For any positive integer $n$, we use $[n]$ to denote the set $\{1, 2, \ldots, n\}$. For any set $S$, let $|S|$ denote its cardinality.
Let $\bX$ denote the vocabulary. We use $\mask$ to denote the mask and extend the vocabulary $\bX$ by including a single point $\{\mask\}$ to obtain $\ol \bX = \bX \cup \{\mask\}$.
For any sequence $x = (x_1, x_2, \dots, x_L) \in \bX^L$ and index set $S \subseteq [L]$, we denote by $x_S = (x_i)_{i \in S} \in \bX^{|S|}$ the subsequence of $x$ at positions in $S$. In addition, let $\cP_S:\bX^L \to \ol\bX^L$ denote the projection defined as
\begin{align}\label{eq:projection}
    \big[\cP_S(x)\big]_i = x_i ~\text{ if } i \in S, ~\text{ and }~ \mask \text{ otherwise}.
\end{align}
For a random variable $X$, we write $p_X$ for its distribution and density interchangeably depending on the context.
For random vectors $X,Y$ with distributions $p_X$ and $p_Y$, we denote by $\KL(p_X\,\|\, p_Y)\defn\int p_X(x)\log\frac{p_X(x)}{p_Y(x)} \diff x$ the Kullback-Leibler divergence between $p_X$ and $p_Y$.

For two functions $f(n),g(n) > 0$, we write $f(n)\lesssim g(n)$ or $f(n)=O\big(g(n)\big)$ if there exists some absolute constant $C>0$ such that $f(n)\leq Cg(n)$. 
Similarly, $f(n)\gtrsim g(n)$ or $f(n)=\Omega\big(g(n)\big)$ means $f(n)\geq C'g(n)$ for some absolute constant $C'>0$.
We write $f(n)\asymp g(n)$ or $f(n)=\Theta\big(g(n)\big)$ if $Cg(n)\leq f(n)\leq C'g(n)$ for some absolute constants $C' > C > 0$. Additionally $\wt O$, $\wt \Omega$, and $\wt \Theta$ are used to denote the corresponding asymptotic notation with logarithmic factors.


%% file: problem.tex
\section{Preliminaries}
\label{sec:problem}

In this section, we review the basics of MDMs and DLMs.

\paragraph{Forward (masking) process.}

Let $X^{(0)} = (X_1^{(0)}, X_2^{(0)}, \dots, X_L^{(0)}) \in \bX^L$ be a token sequence drawn from the data distribution $p_\data$. 
The forward process corrupts $X^{(0)}$ over $K$ steps by progressively replacing tokens with the mask symbol $\mask$ until all indices are masked.

Specifically, let $T = (T^{(1)}, T^{(2)}, \ldots, T^{(K)})$ be a partition of the index set $[L]$, 
i.e., $T^{(i)} \cap T^{(j)} = \varnothing$ for $i\neq j$ and $\cup_{k=1}^K T^{(k)} = [L]$. For each $k\in[K]$, $T^{(k)} \subseteq [L]$ represents the set of token positions newly masked at iteration $k$, and we define $M^{(k)} \defn \cup_{j=1}^k T^{(j)}$ the cumulative set of masked positions up to iteration~$k$.

Let $X^{(k)}$ denote the partially masked sequence after $k$ masking steps.
Specifically, we update the sequence $X^{(k)}$ from $X^{(k-1)}$ by masking the tokens at positions in set $T^{(k)}$:
\begin{align*}
    X^{(k)}_i = \begin{cases}
        \mask, & \text{if } i \in T^{(k)}, \\
        X^{(k-1)}_i, & \text{if } i \notin T^{(k)}.
    \end{cases}
\end{align*}
Equivalently, tokens of the original sequence $X^{(0)}$ at positions in $M^{(k)}$ are masked in $X^{(k)}$, i.e.,
$
    X^{(k)} = \cP_{[L]\setminus M^{(k)}}(X^{(0)}).
$
After $K$ steps, the sequence is fully masked, i.e., $X^{(K)}=(\mask, \mask, \ldots, \mask) \in \ol\bX^L$.

\paragraph{Mask predictor training.}
DLMs generate text by learning to reverse the above forward process, thereby transforming a fully masked sequence into a new sample from $p_\data$. 
This time-reversal is realized by a mask predictor, which aims to learn the conditional distribution of masked tokens given the unmasked ones at each step of the forward process.

As discussed in Section~\ref{sec:intro}, directly modeling the joint conditional distribution of multiple masked tokens can be intractable in high dimensions. Therefore, DLMs adopt a factorized approximation. 
Concretely, for $i\in[L]$, we denote the true conditional marginal of $X_i^{(0)}$ given an arbitrary set $S$ of unmasked tokens (under the data distribution $p_\data$) as
\begin{align}\label{eq:true-conditional}
    p^\star_{i}( x\mid z_S) \defn \bP \big\{X^{(0)}_i =  x\mid X^{(0)}_S = z_S\big\}
\end{align}
for any $S \subseteq [L] \setminus \{i\}$, $z_S\in \bX^{|S|}$, and $x\in \bX$.
The goal is to build an estimator $\wh p_{i}(\cdot \mid \cdot)$ for each conditional marginal $p^\star_i(\cdot \mid \cdot)$ and construct a mask predictor $\wh p = \prod_{i=1}^L\wh p_{i}$.

The mask predictor $\wh p$ is trained via maximizing the log-likelihood of the masked tokens given observed tokens (or equivalently, minimize the KL divergence between the data and sampled distributions):
\begin{align} \label{eq:objective}
\max_{p = \prod_{i=1}^L p_i
} \E_{X^{(0)}, S,\tau}\Bigg[\frac{L}{|M^{(\tau)}|}\sum_{i\in M^{(\tau)}} \log p_i\big(X^{(0)}_{i} \mid X^{(\tau)} \big)\Bigg].
\end{align}
Here, $X^{(0)}$ is a training sample drawn from $p_\data$, $S = (S^{(1)},\dots,S^{(K)})$ is a sequence of unmasking sets  forming a partition of $[L]$, $\tau$ is an iteration index, $M^{(\tau)} = \cup_{k=\tau}^K S^{(k)}$ is the set of masked tokens at step $\tau$, and $X^{(\tau)} = \cP_{[L]\setminus M^{(\tau)}}(X^{(0)})$ is the corresponding partially masked sequence. The expectation is over $X^{(0)}\sim p_\data$, the random unmasking sets $S$ drawn from a certain schedule, and the random time $\tau\in[K]$ sampled according to $\bP\{\tau = k \mid S \} = |S^{(k)}|/L$. 
Equivalently, each unmasking set $S^{(k)}$ corresponds to the masking set $T^{(K-k+1)}$ in the forward process, and the objective is to predict conditional marginals of masked tokens at various intermediate masking levels.
In practice, this objective is approximated by averaging over finite training samples.

As a note, one can verify that the product  of the true conditional marginals $\prod_{i=1}^L p^\star_i$ defined in \eqref{eq:true-conditional} is the optimal solution to the above minimization problem \eqref{eq:objective}.

\paragraph{Reverse (unmasking) process.}

After a mask predictor $\wh p$ is trained, a $K$-step DLM sampling procedure generates a token sequence of length $L$ by iteratively unmasking a masked sequence as follows.

Beginning at step $0$, we initialize the sequence $Y^{(0)} = (\mask, \mask, \ldots, \mask)\in \ol\bX^L$ and the cumulative unmasking set $U^{(0)} = \varnothing$.
Then, for each iteration $k = 1, 2, \ldots, K$, we repeat the following two steps:
\begin{enumerate}
    \item \textit{Unmasking set selection.} We first select a subset $S^{(k)}$ from the currently masked positions $ [L]\setminus U^{(k-1)}$ according to some policy $\pi^{(k)} $, which defines a distribution over all the subsets of $[L] \setminus U^{(k-1)}$.
    \item \textit{Token sampling.} Next, given the previous iterate $Y^{(k-1)}$, we sample the tokens from set $S^{(k)}$ \textit{independently} and \textit{simultaneously}, using the learned mask predictor $\wh p = \prod_{i=1}^L \wh p_i$:
\begin{align}
    \wh Y_i \sim \wh p_{i}\big(\cdot \mid Y^{(k-1)}\big),\quad \forall\,i \in S^{(k)}.
\end{align}
    We then update the sequence $Y^{(k)}$ by filling in the newly sampled tokens:
    \begin{align}
    Y^{(k)}_i = \begin{cases}
            \wh Y_i, & \text{if } i \in S^{(k)}, \\
            Y^{(k-1)}_i, & \text{if } i \notin S^{(k)}, 
        \end{cases}
    \end{align}
    and update the cumulative unmasking set 
    $U^{(k)} = U^{(k-1)} \cup S^{(k)}$.
\end{enumerate}
The final iterate $Y_\out \defn Y^{(K)}$ is outputted as the generated sample.

It is worth noting that given a pre-trained mask predictor $\wh p$ and a fixed sequence of unmasking sets $S = (S^{(1)}, S^{(2)}, \ldots, S^{(K)})$ the distribution of the output sample $Y_\out$ is uniquely determined. 
Therefore, we may write $Y_\out = Y_\out^S$ to emphasize the dependence of the generated sample on a specific sequence of unmasking sets $S$.

More generally, since the unmasking set sequence $S$ is selected according to the unmasking scheme $\pi = (\pi^{(1)}, \pi^{(2)}, \ldots, \pi^{(K)})$, the distribution of the generated sample $Y_\out$ is fully characterized by $\pi$ (along with $\wh p$) through the induced distribution over $S$.

\paragraph{Information-theoretic measures.}
Finally, we review several information-theoretic measures that will be used in this paper. Let $X = (X_1, X_2, \dots, X_L)\in \bX^L$ be a random sequence with joint distribution~$p_X$. 

For any sets $S,T\in[L]$, we use $H(X_S)$ to denote the entropy of the random sequence $X_S = (X_i)_{i \in S}$, and define the conditional entropy of $X_S$ given $X_T$ as
    \begin{align}
        H(X_S \mid X_T) & \defn \bE_{z\sim p_{X_T}}[H(X_S \mid X_T = z)]. 
    \end{align}

The \textit{total correlation} (TC) of $X_S$ is defined as:
    \begin{align}
    \TC(X_S) \defn \sum_{i \in S} H(X_i) - H(X_S).
    \end{align}
    Intuitively, TC measures the global dependence among the variables in $X_S$ by quantifying the distance between the joint distribution $p_{X_S}$ and the fully factorized distribution $\prod_{i \in S} p_{X_i}$.
The \textit{conditional total correlation} of $X_S$ given $X_T$ is given by
    \begin{align}
        \TC(X_S \mid X_T) \defn \sum_{i \in S} H(X_i \mid X_T) - H(X_S \mid X_T).
    \end{align}
The \textit{dual total correlation} (DTC) of $X_S$ is defined as:
    \begin{align}
        \DTC (X_S) \defn H(X_S) - \sum_{i \in S} H \big(X_i \mid X_{S \setminus \{i\}}\big).
    \end{align}
DTC has been shown to characterize how well the distribution can be approximated by a sparse mixture of product distributions~\citep{austin2018multi,chen2025optimal}.
Moreover, when each $X_i$ is determined by the rest, all the conditional entropies are zero, and the DTC equals the entropy, i.e., $\DTC (X_S) = H(X_S)$.

The \textit{conditional dual total correlation} of $X_S$ given $X_T$ is defined as
    \begin{align}
        \DTC (X_S \mid X_T) = H(X_S \mid X_T) - \sum_{i \in S} H\big(X_i \mid X_{(S \setminus \{i\}) \cup T}\big).
    \end{align}

Finally, we use $H(X)$ and $H(p_X)$ interchangeably to denote the entropy of $X$, and likewise for other information-theoretic quantities.

%% file: results.tex
\section{Main results}
\label{sec:results}

In this section, we present our proposed unmasking scheme along with its theoretical guarantee on the sampling error measured by KL divergence.

\subsection{Unmasking scheme}
We propose a distribution-agnostic randomized unmasking procedure, denoted by $\pi = \pi(K, [L])$, that unmasks a length-$L$ sequence over $K$ iterations. 
%
For clarity, we describe the scheme in a slightly more general form \(\pi(K,\cI)\), where \(\cI\subseteq[L]\) represents the set of indices that need to be unmasked.

At a high level, $\pi(K, \cI)$ operates recursively. In the first iteration $t=1$, it samples an unmasking set $S^{(1)}\subseteq \cI$. It then applies the same procedure $\pi(K-1, \cI \setminus S^{(1)})$ to the remaining masked positions $\cI \setminus S^{(1)}$ for the next $K-1$ iterations, and generates the unmasking sets $S^{(2)}, \ldots, S^{(K)}$ sequentially.

The key design choice is how to sample the first unmasking set $S^{(1)}$ from $\cI$. We do this in two stages:
\begin{itemize}
    \item \textbf{Sample a batch size.} We first draw $\ell=|S^{(1)}|$ by assigning weights $w_\ell(K,|\cI|)$ for $\ell \geq 1$.
    \item \textbf{Sample a subset of that size.} Conditioned on $\ell$, we choose $S^{(1)}$ uniformly at random among all $\ell$-subsets of \(\cI\).
\end{itemize}

The batch-size distribution is governed by a scalar coefficient $f(K,L')$ (defined for $K \geq 1$ and $L' \geq K$), which will later appear as the leading factor in our KL sampling error bounds for the unmasking schedule~$\pi$.
%
%
Importantly, for fixed $(K,|\cI|)$, the weights $\{w_\ell(K,|\cI|)\}_{\ell \geq 1}$ depend only on coefficients $f(K-1,\cdot)$ from the previous recursion level $K-1$. As a result, the entire schedule can be computed inductively starting from the base case $K=1$.

\paragraph{General framework.}
We now formalize the general construction of the unmasking scheme $\pi(K, \cI)$ in detail.
This construction provides a template for a family of randomized unmasking schemes, parameterized by the coefficient \(f(K,L')\) and the associated batch-size weights \(\{w_\ell(K,|\cI|)\}_{\ell\ge1}\).
We will later specify these quantities in two concrete ways, leading to a TC-adaptive scheme $\pi_\tc$ and a DTC-adaptive scheme $\pi_\dtc$.

\medskip
\noindent\textbf{Base case ($K = 1$).}
For any index set $ \cI $ with $|\cI| \geq 1$, the scheme $\pi(1, \cI)$ simply unmasks all tokens in $\cI$ in a single iteration, i.e., $\bP\{S^{(1)} = \cI\} = 1$.
In addition, we initialize the coefficient $f(1, L')$ for all $1 \leq L' \leq L$.

\medskip
\noindent\textbf{Recursion ($K \geq 2$).}
Assume that the scheme $\pi(K-1, \cJ)$ has been defined for any index set $\cJ$ with $K-1 \leq |\cJ| \leq L$, and that the coefficient $f(K-1, L')$ has been defined for all $K-1 \leq L' \leq L$.
We now construct $\pi(K, \cI)$ for any index set $\cI$ with $K \leq |\cI| \leq L$. 
For notational convenience, we denote by $L' := |\cI|$ the size of set $\cI$ to be unmasked.
\begin{itemize}
    \item 
\textit{Step 1: sample the first unmasking set $S^{(1)}$.}
    Since we must leave at least $K-1$ positions masked for the remaining $K-1$ iterations, the first batch size must satisfy $|S^{(1)}|\in\{1, 2, \ldots, L'-K+1\}$.
We assign each possible size $\ell$ a nonnegative weight 
\begin{align}\label{eq:wl-def}
w_\ell (K,L'), \quad \forall\,\ell = 1, 2, \ldots, L'-K+1,
\end{align}
which is determined by the coefficient $f(K-1, \cdot)$ at level $K-1$,
and define the normalizing constant 
\begin{align}\label{eq:Psi-def}
\Psi(K,L') \defn \sum_{\ell=1}^{L'-K+1} w_\ell(K,L').
\end{align}
%
Using these weights, we draw the size of the first unmasking set $|S^{(1)}|$ according to
\begin{align}
    \notag
\bP\{|S^{(1)}| = \ell\} = \frac{w_\ell(K,L')}{\Psi(K,L')}, \quad \forall\,\ell = 1, 2, \ldots, L'-K+1.
\end{align}
%
Conditioned on $|S^{(1)}| = \ell$, we then sample $S^{(1)}$ uniformly at random among all subsets of $\cI$ with size~$\ell$.

Finally, we compute the level-$K$ coefficient $f(K,\cdot)$ based on the previously defined $f(K-1,\cdot)$ and the normalizing constant $\Psi(K,\cdot)$ at level $K$.

\item \textit{Step 2: recurse on the remaining indices.}
    Having sampled the first set $S^{(1)}$, we apply the same procedure to the remaining masked set $\cI \setminus S^{(1)}$ for the next $K-1$ iterations.
    Concretely, for $k=2,3,\ldots,K$, we set
    \begin{align}
        \pi^{(k)}(K, \cI) = \pi^{(k-1)}(K-1, \cI \setminus S^{(1)}).
    \end{align}
\end{itemize}




\begin{remark}
The unmasking scheme $\pi(K, \cI)$ is completely specified by the number of iterations $K$ and index set $\cI$. In particular, it does not require prior knowledge of the data distribution $p_{\data}$, estimation of information-theoretic quantities such as TC and DTC, or hyperparameter tuning.
\end{remark}

Unlike procedures that fix the unmasking sizes $|S^{(1)}|,\dots,|S^{(K)}|$ in advance \citep{li2025breaking,chen2025optimal}, our scheme randomizes the unmasking size at each iteration. As we show in the analysis, this randomization is crucial for adapting to the unknown dependence structure of $p_{\data}$ while remaining distribution-agnostic.

\begin{figure}[t]
\centering
\begin{tabular}{cc}
\includegraphics[width=0.425\textwidth]{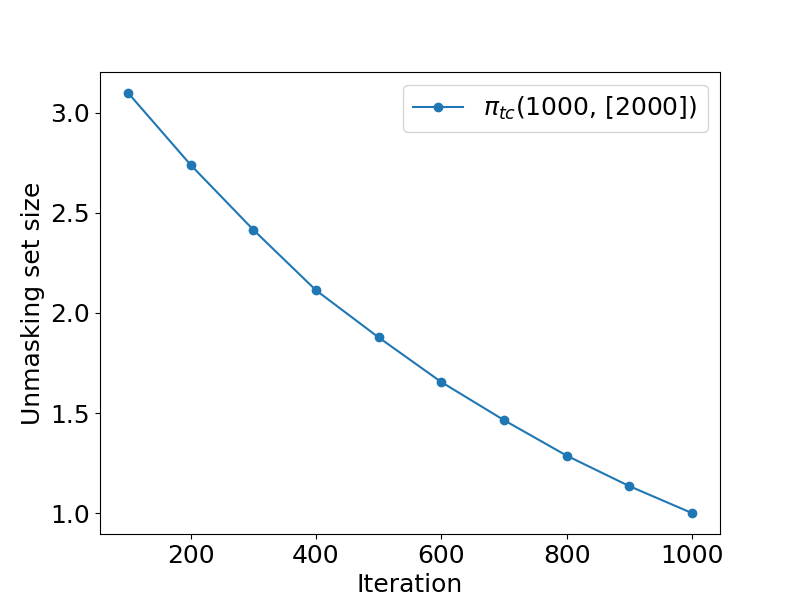}
& \includegraphics[width=0.425\textwidth]{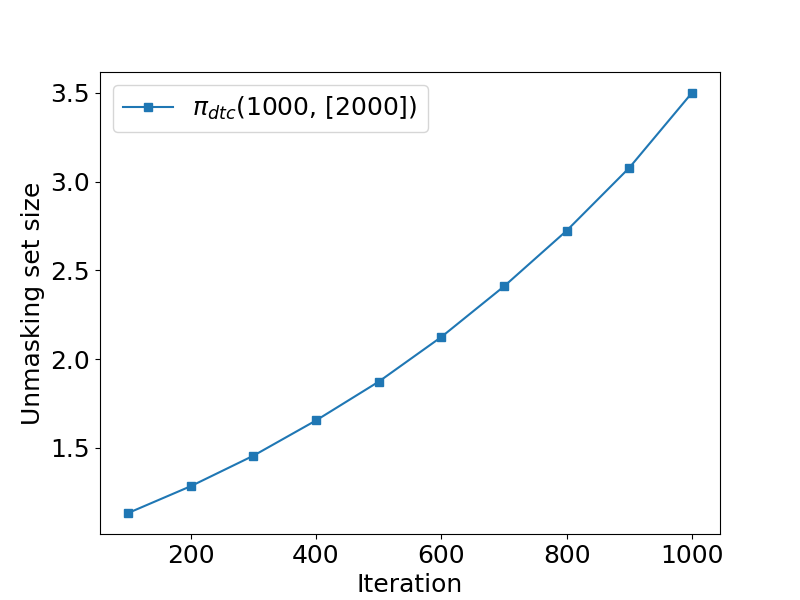}\tabularnewline
 (a) & (b) \tabularnewline
\end{tabular}
\caption{Empirical mean unmasking size vs.~iteration index $k$: (a) TC-adaptive scheme $\pi_\tc$; 
(b) DTC-adaptive scheme $\pi_\dtc$. The total number of iterations is $K = 1000$ and the sequence length is $L = 2000$. \label{fig:unmasking-size}}
\end{figure}
\paragraph{TC/DTC-adaptive schemes.}
With this general construction in place, we now present two concrete unmasking schemes by specifying the coefficients $f(K,L')$ and batch-size weights $\{w_\ell(K,L')\}_{\ell \geq 1}$, which yield sampling convergence guarantees in terms of TC and DTC, respectively.

\begin{itemize}
\item \textbf{TC-adaptive scheme.}
We first introduce the TC-adaptive unmasking scheme, denoted by $\pi_{\tc}(K, \cI)$.

For $K=1$, we initialize the coefficient $f_\tc(1, L')$ by setting
\begin{align}
f_\tc(1, 1) \defn 0, \qquad \text{and} \qquad f_\tc(1, L') \defn 1, \quad \forall\, 2 \leq L' \leq L.
\end{align}

For $K\geq 2$ and $L' \geq K$, we define the batch-size weights by
\begin{subequations}
    \label{eq:wl-tc-def}  
\begin{align}
    w_1^{\tc}(K,L') &\defn 1, \\
    w_\ell^{\tc}(K,L') &\defn \prod_{i = 1}^{\ell-1}  \frac{(L'-i)f_\tc(K-1,L'-i)}{1 + (L' - i - 2)f_\tc(K-1,L'-i-1)}, \quad \forall\,\ell = 2, \ldots, L'-K+1,
\end{align}
\end{subequations}
with the normalizing constant 
\begin{align}\label{eq:Psi-tc-def} 
\Psi_\tc(K,L') \defn \sum_{\ell=1}^{L'-K+1} w_\ell^{\tc}(K,L').
\end{align}
Finally, we update the coefficient $f_\tc(K,L')$ according to
\begin{align}\label{eq:f-def}
f_\tc(K,L') \defn 1-\frac{1 + (L'-2)f_\tc(K-1,L'-1)}{\Psi_\tc(K,L')}.
\end{align}

\item \textbf{DTC-adaptive scheme.}
For the DTC-adaptive scheme, denoted by $\pi_\dtc(K,\cI)$, we initialize the coefficient $f_\dtc(1, L')$ as
\begin{align}
f_\dtc(1, L') \defn \frac{L'-1}{L-L'+1}, \quad \forall\,1 \leq L' \leq L.
\end{align}

For $K\geq 2$ and $L' \geq K$, the batch-size weights are chosen as
\begin{subequations}
    \label{eq:wl-dtc-def}  
\begin{align}
    w_1^{\dtc}(K,L') &\defn 1, \\
    w_\ell^{\dtc}(K,L') &\defn \prod_{i = 1}^{\ell-1} \frac{(L-L'+i)\,f_\dtc(K-1, L'-i)}{1 + (L-L'+i+2)\,f_\dtc(K-1, L'-i-1)},\quad \forall\,\ell = 2, \ldots, L'-K+1,
\end{align}
\end{subequations}
with the normalizing constant 
\begin{align}\label{eq:Psi-dtc-def} 
\Psi_\dtc(K,L') \defn \sum_{\ell=1}^{L'-K+1} w_\ell^{\dtc}(K,L').
\end{align}
The coefficient $f_\dtc(K,L')$ is then updated according to
\begin{align}\label{eq:g-def}
f_\dtc(K,L') \defn -1+\frac{1 + (L - L' + 2)\,f_\dtc(K-1,L'-1)}{\Psi_\dtc(K,L')}.
\end{align}
\end{itemize}

To build some intuition for these two constructed schemes, we plot the empirical mean unmasking size $\bE[|S^{(k)}|]$ at each iteration $k$ in Figure~\ref{fig:unmasking-size}.
As can be seen, the TC-adaptive scheme $\pi_\tc$ unmasks more tokens in early iterations and progressively fewer later on. In contrast, the DTC-adaptive scheme $\pi_\dtc$ starts with smaller unmasking sizes and gradually increases the batch size as sampling proceeds.
Intuitively, $\pi_\tc$ front-loads unmasking because TC primarily penalizes within-batch dependence, while $\pi_\dtc$ back-loads unmasking because low-DTC often reflects a smaller number of degrees of freedom that, once revealed, reduce the conditional dependence of the remaining tokens and allow larger parallel updates in later iterations.

Finally, we briefly discuss the computational cost of sampling unmasking sets
from the proposed scheme.
After a one-time $O(KL)$ dynamic-programming precomputation of $f(K,L')$ for $K \leq L' \leq L$, the batch size $|S^{(k)}|$ can be drawn in $O(1+|S^{(k)}|)$ time per iteration via inverse-transform sampling, and the corresponding set can be sampled uniformly in overall $O(L)$ time (e.g., via reservoir
sampling). Consequently, generating all unmasking sets costs $O(K+L)$ time after precomputation, and thus our strategy introduces essentially no additional overhead.


\subsection{Theoretical guarantee}

We proceed to present the sampling guarantees of the proposed unmasking schemes $\pi_{\tc}$ and $\pi_{\dtc}$, stated in terms of the KL divergence between the data distribution $p_{\data}$ and the distribution $p_{Y_\out}$ of the sampled sequence after $K$ iterations.

Following prior analysis of DLM sampling convergence \citep{li2025breaking,chen2025optimal}, it is useful to separate 
two sources of error: (i) {prediction error} due to using an imperfect mask predictor, and (ii) {intrinsic sampling error} due to parallel decoding, which persists even with access to true conditional marginals.

We begin by quantifying the quality of the learned mask predictor used during the token generation process.
Recall the training objective from \eqref{eq:objective} in Section~\ref{sec:problem},
whose optimal solution $p^\star$ corresponds to the product of the true conditional marginals $p_i^\star$ defined in \eqref{eq:true-conditional}. 
\begin{definition}
For any mask predictor $\wh p= \prod_{i=1}^T\wh p_i$, define its prediction error under an unmasking schedule~$\pi$ as
\begin{align} \label{eq:estimate}
    \veps_{\train}(\pi) & \defn {\bE}_{X^{(0)}, S,\tau}\Bigg[\frac{L}{|M^{(\tau)}|}\sum_{i\in M^{(\tau)}} \log p_i^\star\big(X^{(0)}_{i} \mid X^{(\tau)} \big)\Bigg] - {\bE}_{X^{(0)}, S,\tau}\Bigg[\frac{L}{|M^{(\tau)}|}\sum_{i\in M^{(\tau)}} \log \wh p_i\big(X^{(0)}_{i} \mid X^{(\tau)} \big)\Bigg],
    \end{align}
where 
$M^{(\tau)} = \cup_{k=\tau}^K S^{(k)}$ is the set of masked tokens  at step $\tau$ and $X^{(\tau)} = \cP_{[L]\setminus M^{(\tau)}}(X^{(0)})$ is the corresponding partially masked sequence.
The expectation is taken over the data sample $X^{(0)} \sim p_{\data}$, unmasking sets $S = (S^{(1)}, S^{(2)}, \ldots, S^{(K)}) \sim \pi$, and iteration index $\tau\in[K]$ sampled according to
$\bP\{\tau = k \mid S\} = |S^{(k)}|/{L}$.
\end{definition}
Intuitively, $\varepsilon_\train(\pi)$ measures the average log-likelihood suboptimality of the learned mask predictor $\wh p$ relative to the optimal solution $p^\star$, evaluated on the partially masked sequences encountered under schedule $\pi$. When the optimal mask predictor is used, we have $\varepsilon_\train(\pi) = 0$ for any unmasking strategy $\pi$.

We now state our main results regarding the KL sampling error of our proposed unmasking schemes.
To begin with, Theorem~\ref{thm:main-tc} provides the KL sampling error bound of the unmasking scheme $\pi_{\tc}$. 
\begin{theorem}\label{thm:main-tc}
Given any mask predictor $\wh p$, the KL divergence between the data distribution  and the sampled distribution produced by the unmasking schedule $\pi_{\tc}(K, [L])$ 
satisfies, for any $2 \leq K \leq L$,
\begin{align}\label{eq:main-results-tc}
 \KL(p_{\data} \,\|\, p_{Y_\out}) 
 &\leq \E_{S\sim \pi_{\tc}}\big[\KL(p_{\data} \,\|\, p_{Y_\out^S}\big)\big] \leq \frac{H_{L-K+1} - 1}{K + H_{L-K+1} - 2}  \TC(p_\data) + \veps_{\train}(\pi_{\tc}).
\end{align}
Here, $p_{Y_\out^S} $ is the sampled distribution conditioned on a fixed sequence of unmasking sets $S = (S^{(1)}, S^{(2)}, \ldots, S^{(K)})$,
$H_n \defn \sum_{i=1}^{n} i^{-1}$ denotes the $n$-th harmonic number for $n \geq 1$, and
$\veps_{\train}(\pi_{\tc})$ is the prediction error of the mask predictor $\wh p$ defined in \eqref{eq:estimate}.
\end{theorem}

The next theorem provides an analogous KL sampling error bound for the DTC-adaptive unmasking scheme $\pi_{\dtc}$.
\begin{theorem}\label{thm:main-dtc}
Given any mask predictor $\wh p$, the KL divergence between the data distribution  and the sampled distribution produced by the unmasking schedule $\pi_{\dtc}(K, [L])$ 
satisfies, for any $H_{L-1} < K \leq L$,
\begin{align}\label{eq:main-results-dtc}
 \KL(p_{\data} \,\|\, p_{Y_\out}) 
 &\leq \E_{S\sim \pi_{\dtc}}\big[\KL(p_{\data} \,\|\, p_{Y_\out^S}\big)\big] \leq \frac{H_{L-1}}{K - H_{L-1} }  \DTC(p_\data) + \veps_{\train}(\pi_{\dtc}).
\end{align}
Here, $p_{Y_\out^S} $ is the sampled distribution conditioned on a fixed sequence of unmasking sets $S = (S^{(1)}, S^{(2)}, \ldots, S^{(K)})$, $H_n \defn \sum_{i=1}^{n} i^{-1}$ denotes the $n$-th harmonic number for $n \geq 1$, and
$\veps_{\train}(\pi_{\dtc})$ is the prediction error of the mask predictor $\wh p$.
\end{theorem}

In words, Theorems~\ref{thm:main-tc}--\ref{thm:main-dtc} demonstrate that the KL sampling error decomposes into two components: (i) a dependence measure $\TC(p_\data)$ or $\DTC(p_\data)$, which measures the intrinsic complexity of the target distribution $p_\data$, and (ii) an estimation component $\varepsilon_{\train}$ that captures the suboptimality of the learned mask predictor.
Notably, our convergence theory covers the practically critical regime $K < L$, where multiple tokens are generated per iteration. This provides theoretical support for the key promise of DLMs---parallel decoding with controlled degradation in sampling accuracy.

We now discuss several important implications.
\begin{itemize}
    \item \textit{Iteration complexity.}
Since $\log n \leq H_n \leq 1 + \log n$ for all $n \geq 1$, the sampling convergence rate of $\pi_\tc$ in \eqref{eq:main-results-tc} can be simplified to
\begin{align*}
    \frac{\TC(p_\data) }{K + \log (L-K+1) - 2}\log (L-K+1)    \leq \frac{\TC(p_\data)  }K\log L 
\end{align*}
for any $K \leq  L-7$.
In particular, when the optimal mask predictor is used (i.e., $\veps_{\train} = 0$), the number of sampling iterations required to achieve $\veps$-accuracy in KL divergence is at most
\begin{align*}
 \frac{\TC(p_\data)}{\veps}\log L .
\end{align*}
Similarly, the sampling error bound in \eqref{eq:main-results-dtc} demonstrates that the iteration complexity of $\pi_\dtc$ to achieve $\veps$-accuracy in KL divergence is upper bounded by
\begin{align*}
     \bigg(\frac{\DTC(p_\data) }{\veps} +1 \bigg) \big(\log (L-1)+1\big) 
\end{align*}
As a consequence, when $\min\{\TC(p_\data), \DTC(p_\data)\} \ll L$, our schedule yields a substantial speedup over one-by-one AR decoding and DLM unmasking schemes that requires $\Theta(L)$ iterations.

\item \textit{Adaptivity to intrinsic dependence.}
Notably, such sampling acceleration is achieved without requiring prior knowledge of the data distribution or hyperparameter tuning. This reveals the feasibility of dependence-adaptive DLM sampling that automatically exploits the low-complexity structure of data to achieve faster convergence.

\item \textit{Improvement over fixed-size schedules.} The work of \citet{chen2025optimal} shows that there exists a distribution $p_0$ such that, even with access to the true conditional marginals, any unmasking strategy with a fixed-size schedule requires at least $\Omega(\veps^{-1}\min\{\TC(p_0),\DTC(p_0)\}\log L)$ iterations to achieve $\veps$-accuracy in KL divergence. 
Therefore, our randomized unmasking schemes can be more efficient than fixed-size schedules in the worst case.

\item \textit{Unmasking schedule choice.} Our results suggest a simple guideline for choosing the unmasking schedule:
 $\pi_\tc$ is preferable when TC is the smaller intrinsic complexity measure, while $\pi_\dtc$ is preferable when DTC is smaller. More broadly, the two schedules provide complementary adaptivity guarantees, offering a principled way to match the DLM sampler to the dominant dependence structure of the target distribution.

\end{itemize}

Finally, we discuss representative settings where TC and DTC are significantly smaller than the ambient sequence length $L$, which illustrates the practical relevance of our results.
\begin{itemize}
    \item \textit{Low-TC distributions.}    
Consider the parity (checksum) distribution of $X= (X_1,\dots,X_L)$ where $X_1,\dots,X_{L-1}$ are i.i.d.~$\mathsf{Unif}(\{0,1\})$ random variables and
$X_L = X_1 \oplus X_2 \oplus \cdots \oplus X_{L-1}$. 
It is easy to compute $H(X_i) = 1$ for $i\in[L]$ and $H(X)=L-1$. 
Since each variable is determined by the remaining $L-1$ variables, we have
$\TC(X)=\sum_{i=1}^L H(X_i)-H(X)=1$, $\DTC(X) = H(X) - \sum_{i=1}^L H\big(X_i \mid X_{-i}\big) = L-1$, and $\sum_{i=1}^L I(X_i; X_{-i}) = L$.
This example illustrates that the distribution with a weak ``checksum-like'' constraint has very small TC.
More generally, small TC arises in distributions that are close to independent except for a limited amount of shared information, or a small number of global constraints.

\item \textit{Low-DTC distributions.}
Consider a latent vector $U\in\{0,1\}^d$ with i.i.d.\ $\mathsf{Unif}(\{0,1\})$ entries, and let
\[
X \;=\; G U \in \{0,1\}^L,
\]
where $G\in\{0,1\}^{L\times d}$ has rank $d$ (over $\mathbb{F}_2$). Then $X$ is uniform over a
$d$-dimensional linear subspace of $\{0,1\}^L$, and hence $H(X)=d$. By the definition
$\DTC(X)= H(X)-\sum_{i=1}^L H(X_i\mid X_{-i})$, we immediately obtain
$
\DTC(X)\le H(X)=d.
$
In particular, when $d\ll L$, the distribution has small DTC even though it may exhibit strong
global constraints. For instance, if every coordinate $X_i$ is a nontrivial linear function of $U$
(so that $H(X_i)=1$ for all $i$), then
$
\TC(X)=\sum_{i=1}^L H(X_i)-H(X)=L-d
$
can be large. More generally, small DTC arises in distributions driven by a low-dimensional latent source (e.g., low-rank or manifold structure), where the intrinsic degrees of freedom are much smaller than the sequence length $L$.
We refer to \citet{dmitriev2026efficient} for a detailed discussion of distributions with small DTC.

\end{itemize}

%% file: analysis.tex
\section{Analysis}
\label{sec:analysis}

This section outlines the proof strategy for Theorem~\ref{thm:main-tc} and Theorem~\ref{thm:main-dtc}.
\subsection{Analysis for TC-adaptive scheme (Proof of Theorem~\ref{thm:main-tc})}
\label{sec:analysis-tc}
\paragraph{Step 1: Decoupling prediction and sampling errors.}

We first separate the error due to an imperfectly learned mask predictor from the error that is intrinsic to parallel unmasking.

To this end, we introduce an auxiliary (idealized) sequence $(\overline Y^{(0)}, \overline Y^{(1)},\dots,\overline Y^{(K)})$, which uses the same unmasking strategy $\pi_\tc(K, [L])$ as the practical sampler $Y$, but replace the learned mask predictor $\wh p$ with the optimal mask predictor $p^\star$ (the product of true conditional marginals under $p_\data$). 

Similar to $Y$, we denote by $\ol Y_\out \defn \overline Y^{(K)}$ the output of the idealized sampler and write $\ol Y_\out = \overline Y_\out^S$ when we need to highlight the dependence on the unmasking sets $S = (S^{(1)}, S^{(2)}, \ldots, S^{(K)})$.

As shown in \citet{li2025breaking}, the difference between the expected KL errors of $Y_\out$ and $\overline Y_\out$ is exactly the prediction error $\veps_\train$ defined in \eqref{eq:estimate}:
\begin{align}\label{eq:train-sample-decouple}
\mathbb{E}_{S\sim \pi_\tc}\Big[\mathsf{KL} \big(p_{X^{(0)}} \,\|\, p_{Y_\out^S})\Big] - \mathbb{E}_{S\sim \pi_\tc}\Big[\mathsf{KL}(p_{X^{(0)}} \,\|\, p_{\overline Y_\out^S} \big)\Big]
 = \varepsilon_\train(\pi_\tc),
\end{align}
For completeness, we include a proof in Appendix \ref{sec:proof-train-sample-decouple}.
Consequently, it suffices to analyze the expected KL sampling error of the output of the idealized sampler $\overline Y_\out$.

In what follows, we write $X = X^{(0)}$ and $\ol Y = \ol Y_\out = \overline Y^{(K)}$ for convenience of presentation, so that $p_X = p_\data$.

\paragraph{Step 2: Recursive decomposition of KL error.}

We first notice that when a set of tokens is generated independently using a product of conditional marginals, the resulting discrepancy from the true joint distribution is precisely captured by the TC of those tokens. 
This is formalized in the following lemma, whose proof can be found in Appendix \ref{sec:proof-single_batch_error}.
\begin{lemma}
\label{lem:single_batch_error}
Fix an arbitrary masked index set $\cI \subseteq [L]$ and condition on the already generated tokens $X_{\cI^\mathsf{c}} =x_{\cI^\mathsf{c}}$. For any set $\cS \subseteq \cI$, the KL divergence between the true joint conditional distribution and the sampled distribution of a single-step sampler for tokens in $\cS$ satisfies
\[
    \KL\Big(p_{X_\cS \mid X_{\cI^\mathsf{c}}}(\cdot \mid x_{\cI^\mathsf{c}}) \,\big\|\, \prod_{i \in \cS} p_{X_i \mid X_{\cI^\mathsf{c}}}(\cdot \mid x_{\cI^\mathsf{c}}) \Big) = \TC \big(p_{X_\cS \mid X_{\cI^\mathsf{c}}}(\cdot \mid x_{\cI^\mathsf{c}}) \big).
\]
\end{lemma}

Our key observation is that the expected KL sampling error of $\overline Y$ admits a clean recursion across iterations. After the first batch is unmasked, the remaining $K-1$ steps form the same problem on the remaining positions, conditioned on what has already been unmasked.
This yields the following recursive decomposition; its proof is deferred to Appendix \ref{sec:proof-prop-recursive}.

\begin{lemma}
    \label{prop:recursive}
Fix an arbitrary unmasking scheme $ \pi $ and masked index set $ \cI $ with $K\leq |\cI| \leq L $. 
Conditioned on the already unmasked tokens $ X_{\cI^\mathsf{c}} =x_{\cI^\mathsf{c}}$, we have
\begin{align*}
    & \bE_{S \sim \pi(K, \cI) } \Big[\KL\big(p_{X_{\cI}\mid X_{\cI^\mathsf{c}}}(\cdot \mid x_{\cI^\mathsf{c}})  \,\|\,  p_{\overline Y_{\cI}^S \mid \overline Y_{\cI^\setc}}(\cdot \mid x_{\cI^\mathsf{c}}  ) \big)\Big] \notag\\
    &\quad= 
    \bE_{S^{(1)}} \Big[ \TC \big(p_{X_{S^{(1)}} \mid X_{\cI^\mathsf{c}}}(\cdot\mid x_{\cI^\mathsf{c}}) \big) \Big] \\
     &\qquad
    +\bE_{S^{(1)},x_{S^{(1)}}} \bigg[\bE_{S^{(2:K)}} \Big[\KL \Big(p_{X_{\cI \setminus S^{(1)}} \mid X_{S^{(1)} \cup \cI^\mathsf{c}}}(\cdot \mid x_{S^{(1)} \cup \cI^\mathsf{c}}) \,\big\|\, p_{\overline Y_{\cI \setminus S^{(1)}}^{S^{(2:K)}}  \mid \overline Y_{S^{(1)} \cup \cI^\mathsf{c}}}(\cdot \mid x_{S^{(1)} \cup \cI^\mathsf{c}} ) \Big) \Big] \bigg] ,
\end{align*}
where the expectation on the second line is over $ S^{(1)} \sim \pi^{(1)}(K, \cI) $, $x_{S^{(1)}} \sim p_{X_{S^{(1)}} \mid X_{\cI^\mathsf{c}}}(\cdot \mid x_{\cI^\mathsf{c}})$, and $S^{(2:K)}\defn (S^{(2)}, S^{(3)}, \ldots, S^{(K)}) \sim \pi(K-1, \cI \setminus S^{(1)})$.
\end{lemma}
We remind the reader of the notation that the superscript $S$ in $\ol Y^S$ highlights the dependence on the fixed sequence of unmasking sets $S = (S^{(1)}, S^{(2)}, \ldots, S^{(K)})$. When we apply the schedule $\pi(K,\cI)$ to unmask tokens in $\cI$, $\overline Y_{\cI}^S$ depends on the unmasking sets $S$ and $\overline Y_{\cI \setminus S^{(1)}}^{S^{(2:K)}}$ depends on the unmasking sets $S^{(2:K)}$.

\paragraph{Step 3: Exact characterization of KL error in terms of TC.}
With this recursion in hand, we show that our specific choice of the first-step size distribution in $ \pi_{\tc}  $ makes the overall expected KL sampling error proportional to the TC of the remaining conditional distribution.
Particularly, the proportionality factor is precisely the coefficient $f_\tc(K, |\cI|)$ defined in \eqref{eq:f-def}.

This is formalized in the following lemma, whose proof is deferred to Appendix \ref{sec:proof-tc-f}.
\begin{lemma}\label{lem:tc-f}
For any masked index set $ \cI \subseteq [L]$ with $K\leq |\cI| \leq L$, conditioned on the already unmasked tokens $ X_{\cI^\mathsf{c}} =x_{\cI^\mathsf{c}}$, the unmasking scheme $\pi_{\tc} (K,\cI)$ satisfies
\begin{align}\label{eq:tc-f}
    &\E_{S \sim \pi_{\tc}(K, \cI) } \Big[\KL \Big(p_{X_{\cI} \mid X_{\cI^\setc}}(\cdot \mid x_{\cI^\mathsf{c}}) \,\|\, p_{\overline Y^{S}_{\cI} \mid \overline Y_{\cI^\setc}}(\cdot \mid x_{\cI^\mathsf{c}} ) \Big) \Big] = f_\tc(K, |\cI|) \cdot \TC \big(p_{X_{\cI} \mid X_{\cI^\setc}}(\cdot \mid x_{\cI^\mathsf{c}})\big).
\end{align}
In particular, by setting $\cI = [L]$, we have
\begin{align}\label{eq:tc-f-final}
\E_{S \sim \pi_{\tc}(K, [L]) } \Big[\KL \big(p_{\data} \,\|\, p_{\overline Y^S_\out} \big) \Big] = f_\tc(K, L) \cdot \TC(p_\data).
\end{align}
\end{lemma}

We now briefly outline the proof of Lemma~\ref{lem:tc-f}, which proceeds by induction on the number of iterations $K$.

The base case $K=1$ is immediate since the scheme unmasks all tokens at once, and the KL sampling error reduces to the TC of the conditional distribution.

Now suppose the claim holds for $K-1$. Combining Lemma~\ref{prop:recursive} with the induction hypothesis, we can express the target error, 
$\E_{S \sim \pi_{\tc}(K, \cI)} \big[\KL(p_{X_{\cI}}(\cdot \mid x_{\cI^\mathsf{c}}) \,\|\, p_{\overline Y_{\cI}^S \mid \overline Y_{\cI^\setc}}(\cdot \mid x_{\cI^\mathsf{c}} ) ) \big]$, 
as an expectation over the first unmasking set $S^{(1)}$ consisting of two contributions: (i) the TC incurred when unmasking the first batch, and (ii) the remaining $(K-1)$-step error, which by the induction hypothesis reduces to a scaled TC of the residual problem. Concretely, we have
\begin{align*}
&\E_{S^{(1)}} \Big[ \TC \big(p_{X_{S^{(1)}} \mid X_{\cI^\setc}}(\cdot \mid x_{\cI^\mathsf{c}}) \big)\Big] + f_\tc(K - 1, |\cI| - |S^{(1)}|) \cdot \E_{x_{S^{(1)}} } \Big[\TC\big( p_{X_{\cI \setminus S^{(1)}} \mid X_{S^{(1)} \cup \cI^\mathsf{c}} }(\cdot \mid x_{S^{(1)} \cup \cI^\mathsf{c}})\big) \Big],
\end{align*}
where the second expectation is over $ x_{S^{(1)}} \sim p_{X_{S^{(1)}} \mid X_{\cI^\mathsf{c}}}(\cdot \mid x_{\cI^\mathsf{c}})$.

Recall that our strategy first randomly selects the size $|S^{(1)}|$ and then chooses $S^{(1)}$ uniformly at random among all subsets of that size.
Hence, by the definition of TC, this expectation can be written as a weighted sum of entropies of subsequences of $X$, indexed by the possible first-step sizes $\ell=|S^{(1)}|$:
\begin{align}\label{eq:weighted-sum-TC}
\frac{1}{\Psi_{\tc}(K, |\cI|)}\sum_{\ell = 1}^{|\cI| - K + 1} w_\ell^\tc(K,|\cI|) a_\ell ,
\end{align}
where weights $w_\ell^\tc(K,|\cI|)$ and normalizing constant $\Psi_{\tc}(K,|\cI|)$ are defined in \eqref{eq:wl-tc-def} and \eqref{eq:Psi-tc-def}, respectively, and $a_\ell$ is a linear combination of entropies of $\ell$- and $(\ell+1)$-subsets, given by 
\begin{align*}
& \frac{\ell}{|\cI|} \sum_{i \in \cI} H \big(p_{X_i \mid X_{\cI^\mathsf{c}}}(\cdot \mid x_{\cI^\mathsf{c}}) \big) - f_\tc(K-1, |\cI|-\ell) \cdot H \big(p_{X_{\cI} \mid X_{\cI^\mathsf{c}}}(\cdot \mid x_{\cI^\mathsf{c}}) \big) \\
&\quad  - \alpha_\ell\sum_{\cA\subseteq \cI \,:\,|\cA|=\ell} H \big(p_{X_\cA \mid X_{\cI^\mathsf{c}}}(\cdot \mid x_{\cI^\mathsf{c}}) \big) + \beta_\ell \sum_{\cB\subseteq \cI \,:\,|\cB|=\ell+1} H \big(p_{X_\cB \mid X_{\cI^\mathsf{c}}}(\cdot \mid x_{\cI^\mathsf{c}}) \big)
\end{align*}
for some coefficients $ \alpha_\ell, \beta_\ell$ (exact expressions can be found in Appendix \ref{sec:proof-tc-f}) depending on $f_\tc(K-1, |\cI|-\ell)$, which are used to determine the weights $w_\ell^\tc(K,|\cI|)$.

If the unmasking size $\ell = |S^{(1)}|$ are fixed deterministically, these subset-entropy terms generally do not cancel and one is left with distribution-specific residual terms. In contrast, the chosen weights $\{w^\tc_\ell\}_{\ell \geq 1}$ over possible sizes in our schedule make these subset-entropy contributions telescope across $\ell$, leaving precisely the TC term, 
$$\TC\big(p_{X_{\cI} \mid X_{\cI^\mathsf{c}}}(\cdot \mid x_{\cI^\mathsf{c}})\big) = \sum_{i\in\cI} H\big(p_{X_i \mid X_{\cI^\mathsf{c}}}(\cdot \mid x_{\cI^\mathsf{c}})\big) - H\big(p_{X_{\cI} \mid X_{\cI^\mathsf{c}}}(\cdot \mid x_{\cI^\mathsf{c}})\big),$$
multiplied by the coefficient $f_\tc(K, |\cI|)$. 
Therefore, this telescoping mechanism is the core reason that our randomized unmasking strategy enables distribution-agnostic adaptivity.

This completes the induction and establishes Lemma~\ref{lem:tc-f}.

\paragraph{Step 4: Bounding coefficient $f_\tc(K,L)$.}

By \eqref{eq:tc-f-final}, it remains to bound the coefficient $ f_\tc(K, L) $. This is established by the following lemma; with proof deferred to Appendix \ref{sec:proof-f-bound}.
\begin{lemma}\label{lem:f-bound}
For any $K \geq 2$ and $K \leq L' \leq L$, the coefficient $f_\tc(K, L')$ satisfies
\begin{align}
f_\tc(K, L') \leq \frac{H_{L'-K+1} - 1}{K + H_{L'-K+1} - 2}.\label{eq:f-bound}
\end{align}
where $H_n \defn \sum_{i=1}^{n} i^{-1}$ denotes the $n$-th harmonic number.
In particular, one has
\begin{equation}
    \label{eq:f-bound-final}
    f_\tc(K, L) \leq \frac{H_{L-K+1} - 1}{K + H_{L-K+1} - 2}. 
\end{equation}
\end{lemma}

To conclude, combining \eqref{eq:train-sample-decouple}, \eqref{eq:tc-f-final}, and \eqref{eq:f-bound-final} yields Theorem~\ref{thm:main-tc}.

\subsection{Analysis for DTC-adaptive scheme (Proof of Theorem~\ref{thm:main-dtc})}
\label{sec:analysis-dtc}

We proceed to present the proof of the convergence theory of the DTC-adaptive scheme $\pi_{\dtc}$ in Theorem~\ref{thm:main-dtc}. It follows a similar high-level structure as the TC case in Section~\ref{sec:analysis-tc}.
Therefore, we only highlight the key differences in the analysis, while omitting details that are essentially the same as in the TC case.

In Step 1, we decouple the prediction error from the parallel sampling errors via \eqref{eq:train-sample-decouple}.

In Step 2, we invoke the recursive decomposition of the KL sampling error (Lemma~\ref{prop:recursive}), aiming to relate the KL error to the TC and DTC of the remaining conditional distribution.

The main technical difference in the DTC case arises in the following Step 3.

\paragraph{Step 3: Characterization of KL error in terms of TC and DTC.}

Unlike the TC-adaptive scheme~$\pi_{\tc}$, the expected KL sampling error of the DTC-adaptive scheme $\pi_{\dtc}$ generally cannot be characterized solely in terms of DTC.
This is natural since the one-step error incurred in the first iteration is exactly the TC of the tokens unmasked in the first batch (see Lemma~\ref{lem:single_batch_error}), which generally cannot be tightly controlled by DTC alone.
Instead, we show that the choice of the first-batch size distribution in $\pi_{\dtc}(K,\cI)$ yields an expected KL  error controlled by a linear combination of TC and DTC of the remaining conditional distribution.
Additionally, the proportionality factor is given by $f_\dtc(K, |\cI|)$ defined in \eqref{eq:g-def}.
In particular, when $\cI = [L]$, the TC term vanishes and the KL sampling error becomes proportional to the DTC of the data distribution.

This is summarized by the following lemma; with proof deferred to Appendix \ref{sec:proof-dtc-g}.
\begin{lemma}\label{lem:dtc-g}
For any masked index set $ \cI \subseteq [L]$ with $K\leq |\cI|\leq L$, conditioned on the already unmasked tokens $ X_{\cI^\mathsf{c}} =x_{\cI^\mathsf{c}}$, the DTC-adaptive unmasking scheme $\pi_{\dtc} (K,\cI)$ satisfies
\begin{align}\label{eq:dtc-g}
    &\E_{S \sim \pi_{\dtc}(K, \cI) } \Big[\KL \big(p_{X_{\cI} \mid X_{\cI^\setc}}(\cdot \mid x_{\cI^\mathsf{c}}) \,\|\, p_{\overline Y_{\cI}^S \mid \overline Y_{\cI^\setc}}(\cdot \mid x_{\cI^\mathsf{c}} ) \big) \Big] \notag \\ 
    &\quad\leq f_\dtc(K, |\cI|)\bigg(\frac{L - |\cI|}{|\cI|} \cdot\TC \big(p_{X_{\cI} \mid X_{\cI^\setc}}(\cdot \mid x_{\cI^\mathsf{c}})\big) + \frac{L}{|\cI|} \cdot\DTC \big(p_{X_{\cI} \mid X_{\cI^\setc}}(\cdot \mid x_{\cI^\mathsf{c}})\big) \bigg).
\end{align}
In particular, when $\cI = [L]$, we have
\begin{align}\label{eq:dtc-g-final}
\E_{S \sim \pi_{\dtc}(K, [L]) } \big[\KL(p_{\data} \,\|\, p_{\overline Y^S_\out} ) \big] \leq f_\dtc(K, L) \cdot \DTC(p_\data).
\end{align}
\end{lemma}

The proof strategy of Lemma~\ref{lem:dtc-g} follows the same inductive template as Lemma~\ref{lem:tc-f} for the TC-adaptive scheme. Concretely, we express the target KL error as a weighted sum of entropies of subsequences of $X$. The key difference is that the weights $\{w_\ell^\dtc\}_{\ell \geq 1}$ are designed to make these subset-entropy terms telescope into a linear combination of TC and DTC, rather than collapsing to DTC alone, as one might hope by direct analogy with the TC case.

\paragraph{Step 4: Bounding coefficient $ f_\dtc(K, L) $.}

Equipped with \eqref{eq:dtc-g-final}, we are left to control the coefficient $ f_\dtc(K, L) $. This is established by the following lemma; with proof deferred to Appendix \ref{sec:proof-g-bound}.
\begin{lemma} \label{lem:g-bound}
    For any $K \geq 2$ and $K\leq L' \leq L$, the coefficient $ f_\dtc(K, L') $ satisfies
    \begin{align}
    f_\dtc(K, L') \leq \frac{\min\{L'-K, \frac{L}{K}(H_{L-1} - H_{L-L'} )\}}{L - \min\{L'-K, \frac{L}{K}(H_{L-1} - H_{L-L'} )\}}.\label{eq:g-bound}
    \end{align}
    where $H_n \defn \sum_{i=1}^{n} i^{-1}$ denotes the $n$-th harmonic number and we set $H_0 \defn 0$ for convenience.
    In particular, for $K > H_{L-1}$, we have
    \begin{equation}
        \label{eq:g-bound-final}
        f_\dtc(K, L) \leq \frac{\min\{L'-K, \frac{L}{K}H_{L-1}\}}{L - \min\{L'-K, \frac{L}{K}H_{L-1}\}} \leq \frac{H_{L-1}}{K - H_{L-1}}.
    \end{equation}
\end{lemma}

Putting \eqref{eq:train-sample-decouple}, \eqref{eq:dtc-g-final}, and \eqref{eq:g-bound-final} together finishes the proof of Theorem~\ref{thm:main-dtc}.

%% file: experiment.tex
\section{Numerical experiment}
\label{sec:experiment}
In this section, we evaluate the sampling performance of our proposed unmasking scheme and empirically verify the convergence theory in Theorems~\ref{thm:main-tc}--\ref{thm:main-dtc} via numerical experiments.

We take the vocabulary $ \bX$ to be the finite field $\mathbb{F}_q$, and let $p_{\data}$ be the uniform distribution over a $d$-dimensional Reed-Solomon(RS) code in $\mathbb{F}_q^L$.
Specifically,
\begin{align*}
    p_{\data}^{\text{RS}} (x) = \begin{cases}
    q^{-d}, & \text{if } \exists \text{ a polynomial } p \text{ over } \mathbb{F}_q \text{ with } \\
    & \mathsf{deg}(p) < d \text{ s.t.\ } x = (p(1), \ldots, p(L)), \\
    0, & \text{otherwise}.
    \end{cases}
\end{align*}
For this distribution, the TC and DTC can be calculated explicitly as
$ \TC(p_{\data}^{\text{RS}} ) 
= (L - d) \log q$ and $ \DTC(p_{\data}^{\text{RS}} ) = d \log q $, respectively.
As our focus is on the sampling stage, we use the optimal mask predictor $p^\star$, so that $ \veps_{\train} = 0$. 
Conditioned on any sequence of unmasking sets $S$, we can compute $p_{Y_\out^S}$ and $\KL \big(p_{\data}^{\text{RS}}  \| p_{Y_\out^S}\big)$ in closed form.
The expected KL error is approximated via $ 10^5 $ Monte Carlo simulations over the randomness of $S\sim \pi$.
We set the sequence length $ L = 2000 $ and alphabet size $ q = 2048 $.
For comparison, we also evaluate the fixed-size unmasking scheme that unmasks $ \lceil L / K \rceil $ tokens uniformly at random from the remaining masked positions at each iteration.

\begin{figure}[t]
\centering
\begin{tabular}{cc}
\includegraphics[width=0.425\textwidth]{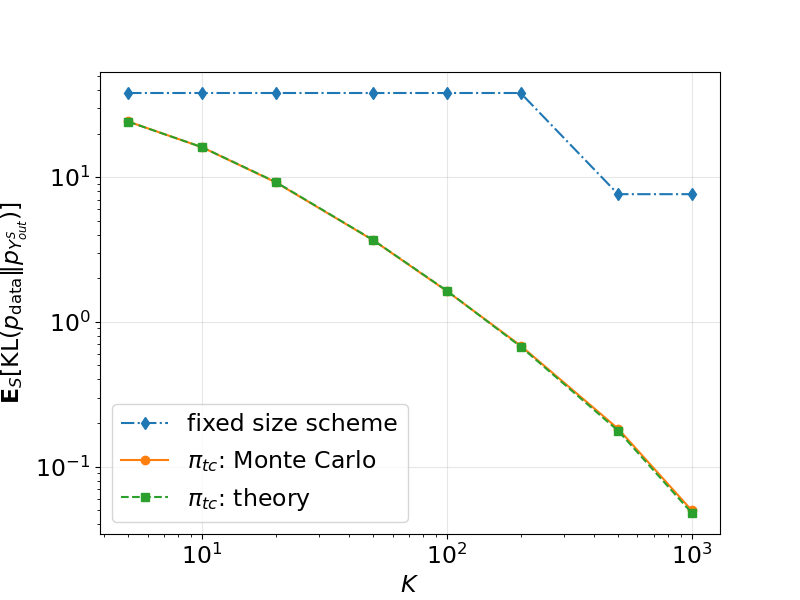}
& \includegraphics[width=0.425\textwidth]{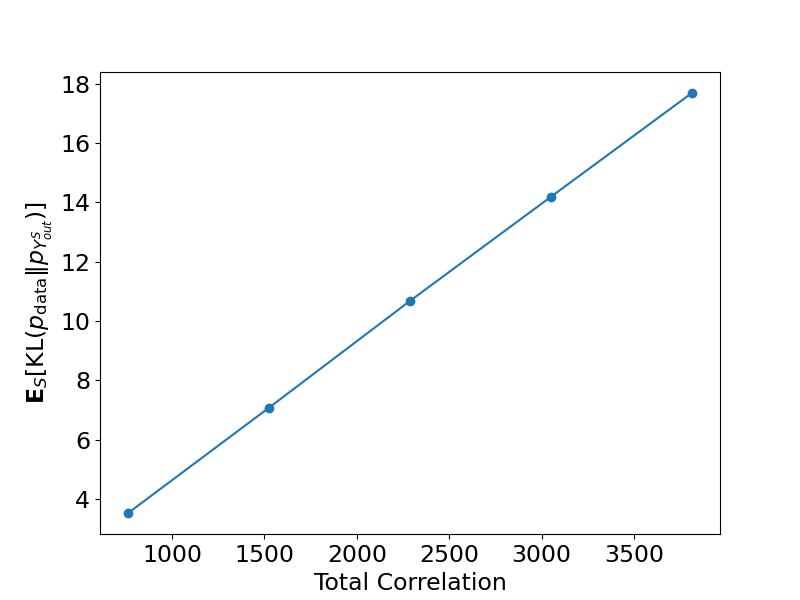}\tabularnewline
 (a) & (b) \tabularnewline
\end{tabular}
\caption{
Expected KL error of the TC-adaptive unmasking scheme $\pi_\tc$:  
(a) KL error vs.~iteration number $K$ for codimension $L-d=5$; 
(b) KL error vs.~TC for number of iterations $K=500$. Sequence length $L=2000$ and alphabet size $q=2048$.\label{fig:sampling}}
\end{figure}
Figure~\ref{fig:sampling} (a) plots the expected KL error of the TC-adaptive scheme $\pi_\tc$ v.s.\ the number of iterations $ K $  when codimension $L - d=5$. 
The empirical error of $\pi_\tc$ closely matches the theoretical prediction in Theorem~\ref{thm:main-tc}, decaying as $\wt O(K^{-1})$ with the number of iterations~$K$.
Moreover, it consistently outperforms the fixed-size unmasking scheme.
Figure~\ref{fig:sampling} (b) plots the expected KL error of $\pi_\tc$ v.s.\ the TC (controlled by the codimension).
As can be seen, the error increases linearly with TC. This is also consistent with our bound, verifying that the proposed schedule $\pi_\tc$ adaptively exploits the global dependence underlying the data distribution that is captured by TC.

\begin{figure}[t]
\centering
\begin{tabular}{cc}
\includegraphics[width=0.425\textwidth]{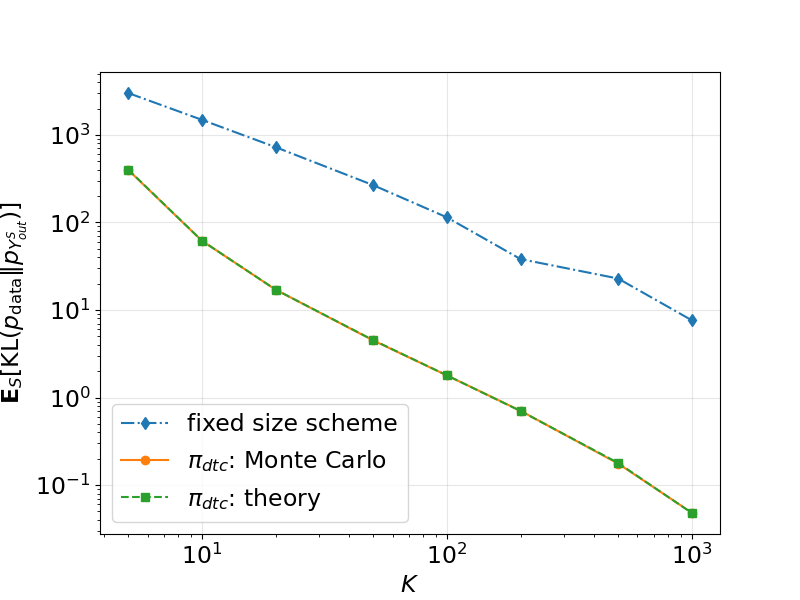}
& \includegraphics[width=0.425\textwidth]{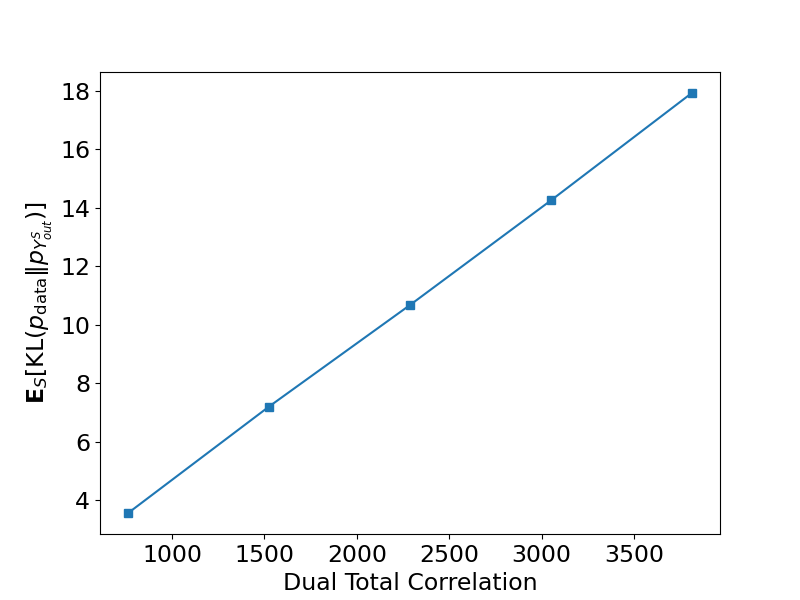}\tabularnewline
 (a) & (b) \tabularnewline
\end{tabular}
\caption{
    Expected KL error of the DTC-adaptive unmasking scheme $\pi_\dtc$:  (a) KL error vs.~iteration number $K$ for dimension $d=5$; 
(b) KL error vs.~DTC for number of iterations $K=500$.
Sequence length $L=2000$ and alphabet size $q=2048$.
\label{fig:sampling-dtc}}
\end{figure}
Figure~\ref{fig:sampling-dtc} (a) presents the expected KL error of the DTC-adaptive scheme $\pi_\dtc$ v.s.\ the number of iterations $ K $ for dimension $d=5$. 
The empirical performance of our unmasking strategy aligns closely with the theoretical decay $\wt O(K^{-1})$ predicted by Theorem~\ref{thm:main-dtc}.
In addition, it consistently surpasses the fixed-size unmasking scheme.
In Figure~\ref{fig:sampling-dtc} (b), we plot the expected KL error of $\pi_\dtc$ v.s.\ the DTC (controlled by the dimension).
The linear increase in the KL error confirms our theoretical bounds and demonstrates that the proposed schedule effectively adapts to the DTC inherent in the data distribution.

Overall, these numerical results confirm our main theoretical findings, demonstrating the effectiveness and adaptivity of our proposed unmasking scheme for distributions with low intrinsic dependence.

%% file: discussion.tex
\section{Discussion}
\label{sec:discussion}

This work has made progress toward a theoretical understanding of how DLMs can exploit intrinsic dependence structure to accelerate sampling.
Our main finding is that, without requiring distribution-specific knowledge or tuning, one can design a distribution-agnostic unmasking schedule that adapts to unknown structure and achieves minimal sampling error governed by TC, an information-theoretic measure of global dependence.
Beyond the convergence theory, our result highlights a concrete algorithmic principle: randomizing the unmasking sizes across iterations can enable structure-adaptive performance.

Before concluding, we point out several directions for future investigation.
First, it is natural to ask whether one can design an adaptive unmasking schedule that is simultaneously optimal in the small-TC and small-DTC regimes. This would sharpen our understanding of sampling acceleration for DLMs and potentially yield a single schedule that performs well across a wider range of distributions.
Second, our theory treats the mask predictor as given, whereas in practice it is trained under a particular masking distribution and then deployed under an inference-time schedule.
Thus developing an end-to-end analysis that couples training with sampling serves as an important research direction.
Third, our schedule chooses which positions to unmask uniformly at random conditioned on the unmasking size, while practical samplers often use state-dependent heuristics (e.g., confidence/entropy) to pick positions.
It would be valuable to characterize the sampling performance of such heuristics theoretically.

%% file: proof-TC.tex
\section{Proof of main lemmas}

\subsection{Notation}

Throughout the appendix, for simplicity of presentation, we denote 
\begin{align*}
\TC(X_{\cA} \mid x_{\cC}) &\defn \TC \big(p_{X_{\cA} \mid X_{\cC}}(\cdot \mid x_{\cC})\big); \\
\TC( X_{\cA }\mid X_{\cB } , x_{\cC}) &\defn \bE_{x_{\cB} \sim p_{X_{\cB} \mid X_{\cC}}(\cdot \mid x_{\cC})} \Big[ \TC\big(p_{X_{\cA} \mid X_{\cB}, X_{\cC}}(\cdot \mid x_{\cB}, x_{\cC}) \big)\Big]; \\
\DTC(X_{\cA} \mid x_{\cC}) &\defn \DTC \big(p_{X_{\cA} \mid X_{\cC}}(\cdot \mid x_{\cC})\big); \\
\DTC( X_{\cA }\mid X_{\cB } , x_{\cC}) &\defn \bE_{x_{\cB} \sim p_{X_{\cB} \mid X_{\cC}}(\cdot \mid x_{\cC})} \Big[ \DTC\big(p_{X_{\cA} \mid X_{\cB}, X_{\cC}}(\cdot \mid x_{\cB}, x_{\cC}) \big)\Big].
\end{align*}
for any index set $\cA,\cB,\cC\subseteq [L]$ and realization $x_{\cC}$ of the already-revealed tokens.

In addition, we use the shorthand notation
\begin{align*}
p_{X_\cA}(\cdot \mid x_{\cB}) \defn p_{X_\cA \mid X_{\cB}}(\cdot \mid x_{\cB})
\end{align*}
for any realization $x_{\cB}$ of the already-revealed tokens and for any index set $\cA $ disjoint from $\cB$.

\subsection{Proof of Lemma \ref{prop:recursive}}
\label{sec:proof-prop-recursive}
We decompose the KL divergence using the chain rule for probability distributions.
For any index set of currently masked tokens $\cI$ with $|\cI| \geq K$, condition on already generated tokens $x_{\cI^\mathsf{c}}$ and the unmasking sets in each iteration $ S = (S^{(1)}, S^{(2)}, \ldots, S^{(K)}) $, the sampling distribution of $ p_{\ol Y_{\cI}^S}(\cdot \mid x_{\cI^\mathsf{c}}) $ factorizes as:
\[
    p_{\ol Y_{\cI}^S}(x_{\cI} \mid x_{\cI^\mathsf{c}}) = \Big(\prod_{i \in S^{(1)}} p_{X_i}(\cdot \mid x_{\cI^\mathsf{c}})\Big)(x_{S^{(1)}}) \cdot p_{\ol Y_{\cI \setminus S^{(1)}}^{S^{(2:K)}}}(x_{\cI \setminus S^{(1)}} \mid x_{S^{(1)} \cup \cI^\mathsf{c}}).
\]
The true conditional distribution can be factorized similarly:
\[
    p_{X_{\cI}}(x_{\cI} \mid x_{\cI^\mathsf{c}}) = p_{X_{S^{(1)}}} (x_{S^{(1)}} \mid x_{\cI^\mathsf{c}}) \cdot p_{X_{\cI \setminus S^{(1)}}}(x_{\cI \setminus S^{(1)}} \mid x_{S^{(1)} \cup \cI^\mathsf{c}}).
\]

Using the chain rule for KL divergence:
\begin{align*}
    &\KL(p_{X_{\cI}}(\cdot \mid x_{\cI^\mathsf{c}}) \| p_{\ol Y_{\cI}^S}(\cdot \mid x_{\cI^\mathsf{c}})) \\
    &= \sum_{x_{\cI}} p_{X_{\cI}}(x_{\cI} \mid x_{\cI^\mathsf{c}}) \log \frac{p_{X_{\cI}}(x_{\cI} \mid x_{\cI^\mathsf{c}})}{p_{\ol Y_{\cI}^S}(x_{\cI} \mid x_{\cI^\mathsf{c}})} \\
    &= \sum_{x_{\cI}} p_{X_{\cI}}(x_{\cI} \mid x_{\cI^\mathsf{c}}) \log \frac{p_{X_{S^{(1)}}}(x_{S^{(1)}} \mid x_{\cI^\mathsf{c}}) p_{X_{\cI \setminus S^{(1)}}}(x_{\cI \setminus S^{(1)}} \mid x_{S^{(1)} \cup \cI^\mathsf{c}})}{(\prod_{i \in S^{(1)}} p_{X_i}(\cdot \mid x_{\cI^\mathsf{c}}))(x_{S^{(1)}}) \cdot p_{\ol Y_{\cI \setminus S^{(1)}}^{S^{(2:K)}}}(x_{\cI \setminus S^{(1)}} \mid x_{S^{(1)} \cup \cI^\mathsf{c}})} \\
    &= \sum_{x_{\cI}} p_{X_{\cI}}(x_{\cI} \mid x_{\cI^\mathsf{c}}) \left( \log \frac{p_{X_{S^{(1)}}}(x_{S^{(1)}} \mid x_{\cI^\mathsf{c}})}{(\prod_{i \in S^{(1)}} p_{X_i}(\cdot \mid x_{\cI^\mathsf{c}}))(x_{S^{(1)}})} + \log \frac{p_{X_{\cI \setminus S^{(1)}}}(x_{\cI \setminus S^{(1)}} \mid x_{S^{(1)} \cup \cI^\mathsf{c}})}{p_{\ol Y_{\cI \setminus S^{(1)}}^{S^{(2:K)}}}(x_{\cI \setminus S^{(1)}} \mid x_{S^{(1)} \cup \cI^\mathsf{c}})} \right) \\
    &= \KL\big(p_{X_{S^{(1)}}}(\cdot \mid x_{\cI^\mathsf{c}}) \| \prod_{i \in S^{(1)}} p_{X_i}(\cdot \mid x_{\cI^\mathsf{c}})\big) + \E_{x_{S^{(1)}} \sim p_{X_{S^{(1)}}}(\cdot \mid x_{\cI^\mathsf{c}})}\left[\KL\big(p_{X_{\cI \setminus S^{(1)}}}(\cdot \mid x_{S^{(1)} \cup \cI^\mathsf{c}}) \| p_{\ol Y_{\cI \setminus S^{(1)}}^{S^{(2:K)}}}(\cdot \mid x_{S^{(1)} \cup \cI^\mathsf{c}})\big)\right] \\
    &= \TC(X_{S^{(1)}}\mid x_{\cI^\mathsf{c}}) + \E_{x_{S^{(1)}} \sim p_{X_{S^{(1)}}}(\cdot \mid x_{\cI^\mathsf{c}})}\left[\KL\big(p_{X_{\cI \setminus S^{(1)}}}(\cdot \mid x_{S^{(1)} \cup \cI^\mathsf{c}}) \| p_{\ol Y_{\cI \setminus S^{(1)}}^{S^{(2:K)}}}(\cdot \mid x_{S^{(1)} \cup \cI^\mathsf{c}})\big)\right], 
\end{align*}
where the last equality follows from the application of Lemma \ref{lem:single_batch_error} to the first term and its proof is deferred to Section \ref{sec:proof-single_batch_error}.

Taking expectation over the random choice of $ S = (S^{(1)}, S^{(2)}, \ldots, S^{(K)}) $ according to $ \pi(K, \cI)$ yields the desired result:
\begin{align*}
    & \bE_{S} \big[\KL(p_{X_{\cI}}(\cdot \mid x_{\cI^\mathsf{c}}) \,\|\, p_{\ol Y_{\cI}^S}(\cdot \mid x_{\cI^\mathsf{c}}  ))\big]
     = \bE_{S^{(1)}} \big[ \TC(X_{S^{(1)}}\mid x_{\cI^\mathsf{c}}) \big] \notag\\ 
    & \quad + \bE_{S^{(1)},\,x_{S^{(1)}}} \bE_{S^{(2:K)}} \big[\KL \big(p_{X_{\cI \setminus S^{(1)}}}(\cdot \mid x_{S^{(1)} \cup \cI^\mathsf{c}}) \| p_{\ol Y_{\cI \setminus S^{(1)}}^{S^{(2:K)}}}(\cdot \mid x_{S^{(1)} \cup \cI^\mathsf{c}} ) \big) \big],
\end{align*}
where $S^{(2:K)} \defn (S^{(2)}, S^{(3)}, \ldots, S^{(K)})$, and the expectation is over $ S \sim \pi(K, \cI) $, $ S^{(1)} \sim \pi^{(1)}(K, \cI) $, $x_{S^{(1)}} \sim p_{X_{S^{(1)}}}(\cdot \mid x_{\cI^\mathsf{c}})$, and $S^{(2:K)} \sim \pi(K-1, \cI \setminus S^{(1)})$.

\subsection{Proof of Lemma~\ref{lem:tc-f}}
\label{sec:proof-tc-f}

The proof of Lemma \ref{lem:tc-f} proceeds by induction on the iteration count $K$. Consider an arbitrary set $\cI$ of currently masked token indices satisfying $|\cI| \geq K$, and condition on the already-revealed tokens $x_{\cI^\mathsf{c}}$. Our goal is to establish that
\begin{align*}
    \E_{S \sim \pi_{\tc}(K, \cI)} \big[\KL(p_{X_{\cI}}(\cdot \mid x_{\cI^\mathsf{c}}) \,\|\, p_{\ol Y_{\cI}^S}(\cdot \mid x_{\cI^\mathsf{c}} ) ) \big] = f_\tc(K, |\cI|) \TC \big(X_{\cI}\mid x_{\cI^\mathsf{c}} \big).
\end{align*}
For notational convenience, we write $ L' \defn |\cI| $ for the remainder of this proof.
Every distribution appearing below is implicitly conditioned on $x_{\cI^\mathsf{c}}$.

\subsubsection{Base case} For $ K = 1 $, the schedule $ \pi_{\tc}(1, \cI) $ reveals every token in $\cI$ at once. 
The resulting sampling distribution therefore factorizes into a product of conditional marginals. Invoking Lemma \ref{lem:single_batch_error} gives
\begin{align*}
    \E_{S \sim \pi_{\tc}(1, \cI)}[\KL(p_{X_{\cI}}(\cdot \mid x_{\cI^\mathsf{c}}) \| p_{\ol Y_{\cI}^S}(\cdot \mid x_{\cI^\mathsf{c}})  )] &= \KL\big(p_{X_{\cI}}(\cdot \mid x_{\cI^\mathsf{c}}) \| \prod_{i \in \cI} p_{X_i}(\cdot \mid x_{\cI^\mathsf{c}})   \big) \\
    &= \TC(X_{\cI} \mid x_{\cI^\mathsf{c}}) = f_\tc(1, L') \TC(X_{\cI} \mid x_{\cI^\mathsf{c}}),
\end{align*}
where the final step uses the fact that $ f_\tc(1, L') = 1 $ whenever $ L' \geq 2 $, and both sides vanish when $ L'=1 $.

\subsubsection{Inductive step} Suppose the statement has been verified for $ K - 1$ iterations. We now extend it to $ K $.

\paragraph{Step 1: Write KL error as a weighted sum $\sum_{\ell} w^\tc_{\ell} a_\ell$.}
Applying Lemma \ref{prop:recursive} with the conditioning on $x_{\cI^\mathsf{c}}$, the expected KL divergence incurred by $ \pi_{\tc}(K, \cI) $ decomposes as:
\begin{align}
    & \bE_{S \sim \pi_{\tc}(K, \cI)} \Big[\KL\big(p_{X_{\cI}}(\cdot \mid x_{\cI^\mathsf{c}}) \| p_{\ol Y_{\cI}^S}(\cdot \mid x_{\cI^\mathsf{c}}) \big)\Big]
     = \bE_{S^{(1)} \sim \pi^{(1)}_\TC(K, \cI)} \big[ \TC(X_{S^{(1)}}\mid x_{\cI^\mathsf{c}}) \big] \notag\\ 
    & \quad + \bE_{S^{(1)} \sim \pi^{(1)}_\TC(K, \cI),\,x_{S^{(1)}} \sim p_{X_{S^{(1)}}}(\cdot \mid x_{\cI^\mathsf{c}})} \Big[ \bE_{S^{(2:K)} \sim \pi_{\tc}(K-1, \cI \setminus S^{(1)})} \big[\KL \big(p_{X_{\cI \setminus S^{(1)}}}(\cdot \mid x_{S^{(1)} \cup \cI^\mathsf{c}}) \| p_{\ol Y_{\cI \setminus S^{(1)}}^{S^{(2:K)}}}(\cdot \mid x_{S^{(1)} \cup \cI^\mathsf{c}} ) \big) \big] \Big] \notag \\
    &= \E_{S^{(1)} \sim \pi^{(1)}_\TC(K, \cI), x_{S^{(1)}} \sim p_{X_{S^{(1)}}}(\cdot \mid x_{\cI^\mathsf{c}})} \big[ \TC(X_{S^{(1)}}\mid x_{\cI^\mathsf{c}} ) + f_\tc(K - 1, L' - |S^{(1)}|) \TC(X_{\cI \setminus S^{(1)}} \mid x_{S^{(1)} \cup \cI^\mathsf{c}}) \big] \notag \\
    &= \E_{S^{(1)}  \sim \pi^{(1)}_\tc (K, \cI)} \big[ \TC(X_{S^{(1)} } \mid x_{\cI^\mathsf{c}}) + f_\tc(K - 1, L' - |S^{(1)}|) \TC( X_{\cI \setminus S^{(1)} }\mid X_{S^{(1)} } , x_{\cI^\mathsf{c}}) \big],\label{eq:KL-decomp-TC}
\end{align}
where the first equality comes from Lemma \ref{prop:recursive}, the second applies the inductive assumption, and the third exploits the definitions of conditional total correlation and conditional entropy.

To evaluate the right-hand side of \eqref{eq:KL-decomp-TC}, recall the two-stage construction of the first unmasking set $S^{(1)}$: its cardinality $\ell = |S^{(1)}|$ is drawn with probability $w^\tc_{\ell}(K,L') / \Psi^\tc(K,L')$ for $\ell = 1, 2, \ldots, L'-K+1$, after which $S^{(1)}$ is selected uniformly at random from all size-$\ell$ subsets of $\cI$.
Accordingly, we introduce $a_\ell = a_\ell(K,L')$, the average over all such size-$\ell$ subsets, for each $1 \leq \ell \leq L'-K+1$:
\begin{align}
    a_\ell(K,L') \defn\frac{1}{\binom{L'}{\ell}}\sum_{S^{(1)} \subseteq \cI: |S^{(1)}|=\ell} \Big( \TC \big(X_{S^{(1)}}\mid x_{\cI^\mathsf{c}}\big) + f_\tc(K-1, L' - \ell)\TC(X_{\cI \setminus S^{(1)}} \mid X_{S^{(1)}}, x_{\cI^\mathsf{c}}) \Big). \label{eq:al-TC-def}
\end{align}
The expected KL sampling error from \eqref{eq:KL-decomp-TC} can then be written as the following weighted sum
\begin{align}\label{eq:KL-decomp-TC-2}
    \bE_{S \sim \pi_{\tc}(K, \cI)} \Big[\KL\big(p_{X_{\cI}}(\cdot \mid x_{\cI^\mathsf{c}}) \| p_{\ol Y_{\cI}^S}(\cdot \mid x_{\cI^\mathsf{c}}) \big)\Big] = \sum_{\ell = 1}^{L'-K+1} \bP\{|S|=\ell\} \cdot a_\ell = 
\frac{1}{\Psi^\tc(K,L')}\sum_{\ell = 1}^{L'-K+1} w^\tc_{\ell} a_\ell.
\end{align}
Our task therefore reduces to bounding the weighted sum $ \sum_{\ell=1}^{L'-K+1} w^\tc_{\ell} a_{\ell} $.

\paragraph{Step 2: Express $\sum_{\ell} w^\tc_{\ell} a_{\ell}$ as a telescoped sum.}
Expanding the total correlation and conditional total correlation in their entropic forms, we can reformulate $a_\ell$ from \eqref{eq:al-TC-def} as:
\begin{align}
    a_\ell &\numpf{i}{=} \frac{1}{\binom{L'}{\ell}}\sum_{S^{(1)} \subseteq \cI: |S^{(1)}|=\ell} \Bigg( \sum_{i \in S^{(1)}} H(X_i \mid x_{\cI^\mathsf{c}}) - H(X_{S^{(1)}} \mid x_{\cI^\mathsf{c}}) \notag \\
    &\quad \quad + f_\tc(K-1, L' - \ell) \bigg( \sum_{j \in \cI \setminus S^{(1)}} H(X_j \mid X_{S^{(1)}}, x_{\cI^\mathsf{c}}) - H(X_{\cI \setminus S^{(1)}} \mid X_{S^{(1)}}, x_{\cI^\mathsf{c}}) \bigg) \Bigg)\notag \\
    &\numpf{ii}{=} \frac{1}{\binom{L'}{\ell}}\sum_{S^{(1)} \subseteq \cI: |S^{(1)}|=\ell} \Bigg( \sum_{i \in S^{(1)}} H(X_i \mid x_{\cI^\mathsf{c}}) - H(X_{S^{(1)}} \mid x_{\cI^\mathsf{c}}) \notag \\
    &\quad \quad + f_\tc(K-1, L' - \ell) \bigg( \sum_{j \in \cI \setminus S^{(1)}} H(X_{\{j\} \cup S^{(1)}} \mid x_{\cI^\mathsf{c}}) - (L' - \ell - 1) H(X_{S^{(1)}} \mid x_{\cI^\mathsf{c}}) - H(X_{\cI} \mid x_{\cI^\mathsf{c}}) \bigg) \Bigg)\notag \\
    &\numpf{iii}{=} \frac{\ell}{L'} \sum_{i \in \cI} H(X_i \mid x_{\cI^\mathsf{c}}) - f_\tc(K-1, L'-\ell) H(X_{\cI} \mid x_{\cI^\mathsf{c}}) \notag \\
    &\quad - \underbrace{\frac{1 + (L' - \ell - 1) f_\tc(K - 1, L' - \ell)}{\binom{L'}{\ell}}}_{=:\alpha_\ell}\sum_{S^{(1)} \subseteq \cI: |S^{(1)}|=\ell} H(X_{S^{(1)}} \mid x_{\cI^\mathsf{c}}) \notag\\
    & \quad + \underbrace{\frac{(\ell + 1) f_\tc(K-1, L'-\ell)}{\binom{L'}{\ell}}}_{=:\beta_\ell} \sum_{T \subseteq \cI: |T|=\ell+1} H(X_{S^{(1)} } \mid x_{\cI^\mathsf{c}}). \label{eq:al-TC}
\end{align}
Here, (i) expands the total correlation and conditional total correlation into sums of entropies; (ii) applies the standard identity $H(X\mid Y) = H(X,Y) - H(Y)$; (iii) relies on a double-counting argument: $\sum_{S^{(1)} \subseteq \cI: |S^{(1)}|=\ell}\sum_{j \in \cI \setminus S^{(1)}} H(X_{\{j\} \cup S^{(1)}} \mid x_{\cI^\mathsf{c}}) = (\ell + 1) \sum_{S^{(1)}  \subseteq \cI: |S^{(1)} |=\ell+1} H(X_T \mid x_{\cI^\mathsf{c}})$, because each $(\ell+1)$-element set $T$ is obtained by adjoining any one of its $\ell+1$ members to its complement subset.

For the entropy sums in $a_\ell$ to cancel telescopically upon forming the weighted combination $\sum_{\ell = 1}^{L'-K+1} w^\tc_{\ell} a_\ell$, we require that the weights $w^\tc_{\ell}$ satisfy
\begin{align*}
w^\tc_{\ell} \beta_{\ell} = w^\tc_{\ell + 1} \alpha_{\ell + 1}, \quad \forall\, l = 1, 2, \ldots, L'-K.
\end{align*}
This telescoping requirement yields the recurrence
\begin{align}
w^\tc_{\ell + 1} = w^\tc_{\ell} \frac{(L'-\ell)  f_\tc(K-1, L'-\ell)}{1 + (L' - \ell - 2)  f_\tc(K - 1, L' - \ell - 1)}, \quad \forall\,l = 1, 2, \ldots, L'-K.
\end{align}
Initializing with $ w_1 = 1 $ and iterating, we arrive at
\begin{align}\label{eq:wl-TC}
w^\tc_{\ell} = \prod_{i=1}^{\ell - 1} \frac{(L'-i)f_\tc(K-1, L' - i)}{1 + (L'-i-2)f_\tc(K-1, L' - i - 1)},\quad \forall\,l = 1,2,\ldots, L'-K+1, 
\end{align}
which coincides precisely with the weight formula specified in \eqref{eq:wl-tc-def}.

\paragraph{Step 3: Compute entropy coefficients.}

Having established \eqref{eq:al-TC}, the next task is to determine the entropy coefficients that appear in the weighted sum $ \sum_{\ell = 1}^{L'-K+1} w^\tc_{\ell} a_\ell $.

\begin{itemize}
    \item 

We start by evaluating the coefficient multiplying $-H(X_{\cI} \mid x_{\cI^\mathsf{c}})$, namely $\sum_{\ell = 1}^{L'-K+1} w^\tc_{\ell} f_\tc(K-1, L'-\ell)$. The key observation is the following rearrangement of $f_\tc(K-1, L'-\ell)$:
\begin{align*}
&f_\tc(K-1, L'-\ell) = \\
&\frac{(L'-\ell)f_\tc(K-1, L'-\ell)}{1 + (L'-\ell-2)f_\tc(K-1, L'-\ell-1)} \Big(1 + (L'-\ell-2)f_\tc(K-1, L'-\ell-1) \Big) - (L'-\ell-1)f_\tc(K-1, L'-\ell) .
\end{align*}
Applying this identity, the coefficient evaluates to:
\begin{align}
    &\sum_{\ell = 1}^{L'-K+1} w^\tc_{\ell} f_\tc(K-1, L'-\ell) \numpf{i}{=} \sum_{\ell = 1}^{L'-K} w^\tc_{\ell} f_\tc(K-1, L'-\ell) \notag\\
    &\numpf{ii}{=} \sum_{\ell = 1}^{L'-K} w^\tc_{\ell}\bigg( \frac{(L'-\ell)f_\tc(K-1, L'-\ell)}{1 + (L'-\ell-2)f_\tc(K-1, L'-\ell-1)} - (L'-\ell-1)f_\tc(K-1, L'-\ell) \notag\\
    &\qquad \qquad\qquad\quad + \frac{(L'-\ell)f_\tc(K-1, L'-\ell)}{1 + (L'-\ell-2)f_\tc(K-1, L'-\ell-1)}(L'-\ell-2)f_\tc(K-1, L'-\ell-1)\bigg) \notag\\
    &\numpf{iii}{=} \sum_{\ell = 1}^{L'-K} w^\tc_{\ell + 1} - \sum_{\ell = 1}^{L'-K} \Big( w^\tc_{\ell} (L'-\ell-1)f_\tc(K-1, L'-\ell) - w^\tc_{\ell + 1} \big(L'-(\ell + 1)-1\big)f\big(K-1, L'-(\ell + 1)\big) \Big) \notag\\
    &= \sum_{\ell = 1}^{L'-K} w^\tc_{\ell + 1} - \Big(w_1 (L'-2)f_\tc(K-1, L'-1) - w^\tc_{L'-K+1} (K-2)f_\tc(K-1, K-1) \Big) \notag\\
    &\numpf{iv}{=} \sum_{\ell = 1}^{L'-K+1} w^\tc_{\ell} -1 - (L'-2)f_\tc(K-1,L'-1).\label{eq:coeff-TC-all}
\end{align}
Here, (i) follows from the vanishing property $f_\tc(K,K)=0$ for every $K\geq 1$; (ii) substitutes the rearrangement of $f_\tc(K-1, L'-\ell)$ derived above;
(iii) invokes the weight definition in \eqref{eq:wl-TC} and reorganizes into a telescoping sum; (iv) applies the initialization $ w_1 = 1 $ together with $f_\tc(K,K)=0$.

\item It remains to handle the coefficient of $ \sum_{i \in \cI} H(X_i \mid x_{\cI^\mathsf{c}}) $ appearing in \eqref{eq:al-TC}. Using the expression for $\alpha_1$ from \eqref{eq:al-TC}, the coefficient can be computed as:
\begin{align} 
    &\sum_{\ell = 1}^{L'-K+1} w^\tc_{\ell} \frac{\ell}{L'} - w_1\frac{1+(L'-2)f_\tc(K-1, L'-1)}{L'} \notag \\
    &= \sum_{\ell = 1}^{L'-K+1} w^\tc_{\ell} - \frac{1}{L'} \sum_{\ell = 1}^{L'-K+1} w^\tc_{\ell} (L'-\ell) - w_1\frac{1+(L'-2)f_\tc(K-1, L'-1)}{L'} \notag \\
    &\stackrel{\text{(i)}}{=} \sum_{\ell = 1}^{L'-K+1} w^\tc_{\ell} \notag \\
    &\quad - \frac{1}{L'} \sum_{\ell = 1}^{L'-K} \Big( w^\tc_{\ell} (L'-\ell)\big(1 + (L'-\ell-1)f_\tc(K-1, L'-\ell)\big) - w^\tc_{\ell} (L'-\ell)(L'-\ell-1)f_\tc(K-1, L'-\ell) \Big) \notag \\
    &\quad -\frac{1}{L'}w^\tc_{L'-K+1} (K-1) - w_1\frac{1+(L'-2)f_\tc(K-1, L'-1)}{L'} \notag \\
    &\stackrel{\text{(ii)}}{=} \sum_{\ell = 1}^{L'-K+1} w^\tc_{\ell} - \frac{1}{L'} \sum_{\ell = 1}^{L'-K} \Big( w^\tc_{\ell} (L'-\ell)\big(1 + (L'-\ell-1)f_\tc(K-1, L'-\ell)\big) \notag \\
    &\qquad \qquad \qquad \qquad \qquad - w^\tc_{\ell + 1} \big(L'-(\ell + 1)\big)\big[1 + \big(L'-(\ell + 1)-1\big)f\big(K-1, L'-(\ell + 1)\big)\big] \Big) \notag \\
    &\quad -\frac{1}{L'}w^\tc_{L'-K+1} (K-1) - w_1\frac{1+(L'-2)f_\tc(K-1, L'-1)}{L'} \notag \\
    &\stackrel{\text{(iii)}}{=} \sum_{\ell = 1}^{L'-K+1} w^\tc_{\ell} \notag \\
    &\quad - \frac{1}{L'} \Big( w_1 (L'-1)\big(1 + (L'-2)f_\tc(K-1, L'-1)\big) - w^\tc_{L'-K+1} (K-1)\big(1 + (K-2)f_\tc(K-1, K-1)\big) \Big) \notag \\
    &\quad -\frac{1}{L'}w^\tc_{L'-K+1} (K-1) - w_1\frac{1+(L'-2)f_\tc(K-1, L'-1)}{L'} \notag \\
    &\stackrel{\text{(iv)}}{=} \sum_{\ell = 1}^{L'-K+1} w^\tc_{\ell} - \big(1 + (L'-2)f_\tc(K-1, L'-1)\big). \label{eq:coeff-TC-single}
\end{align}
Here, (i) decomposes $ (L'-\ell) $ as \[
 (L'-\ell)(1 + (L'-\ell+1)f_\tc(K-1, L'-\ell)) - (L'-\ell)(L'-\ell+1)f_\tc(K-1, L'-\ell); \]
(ii) employs the weight recurrence from \eqref{eq:wl-TC}, which gives
\[
    w^\tc_{\ell} (L'-\ell)f_\tc(K-1, L'-\ell) = w^\tc_{\ell + 1} \big[1 + \big(L'-(\ell + 1)-1\big)f\big(K-1, L'-(\ell + 1)\big)\big];
\] 
(iii) collapses the resulting telescoping sum; (iv) substitutes $ w_1 = 1 $ and leverages $ f_\tc(K,K) = 0 $ for all $ K \geq 1 $.

\end{itemize}

\paragraph{Step 4: Putting everything together.}
Assembling the coefficient formulas \eqref{eq:coeff-TC-all} and \eqref{eq:coeff-TC-single} and inserting them back into \eqref{eq:al-TC}, all intermediate entropy sums cancel telescopically, and only the total correlation over $ \cI $ given $ x_{\cI^\mathsf{c}} $ remains:
\begin{align}
\sum_{\ell = 1}^{L'-K+1} w^\tc_{\ell} a_\ell
& = \bigg(\sum_{\ell = 1}^{L'-K+1} w^\tc_{\ell} - 1 - (L'-2)f_\tc(K-1, L'-1)\bigg) \bigg(\sum_{i \in \cI} H(X_i \mid x_{\cI^\mathsf{c}}) - H(X_{\cI} \mid x_{\cI^\mathsf{c}})\bigg) \notag\\
& = \bigg(\sum_{\ell = 1}^{L'-K+1} w^\tc_{\ell} - 1 - (L'-2)f_\tc(K-1, L'-1)\bigg) \TC(X_{\cI} \mid x_{\cI^\mathsf{c}}). \label{eq:weighted-sum-tc}
\end{align}

Substituting \eqref{eq:weighted-sum-tc} into the KL error expression \eqref{eq:KL-decomp-TC-2} yields
\begin{align}
    &\bE_{S \sim \pi_{\tc}(K, \cI)} \Big[\KL\big(p_{X_{\cI}}(\cdot \mid x_{\cI^\mathsf{c}}) \| p_{\ol Y_{\cI}^S}(\cdot \mid x_{\cI^\mathsf{c}}) \big)\Big] = \frac{1}{\Psi^{\tc}(K,L')}\sum_{\ell = 1}^{L'-K+1} w^\tc_{\ell} a_\ell \notag\\
     &\qquad = \bigg(1-\frac{1 + (L'-2)f_\tc(K-1, L'-1)}{\Psi^{\tc}(K, L')}\bigg) \TC(X_{\cI} \mid x_{\cI^\mathsf{c}}) \notag\\
     &\qquad = f_\tc(K, L') \TC(X_{\cI} \mid x_{\cI^\mathsf{c}}), \label{eq:KL-bound-final-TC}
\end{align}
with the final step following from the definition of $ f_\tc(K, L') $ in \eqref{eq:f-def}.
This closes the inductive argument and finishes the proof of Lemma \ref{lem:tc-f}.

%% file: proof-TC-f.tex

\subsection{Proof of Lemma \ref{lem:f-bound}}
\label{sec:proof-f-bound}
By the definition of $ f_\tc(K, L') $ \eqref{eq:f-def}, we only need to prove 
\begin{equation}
    \frac{\Psi^{\tc}(K, L')}{1 + (L'-2)f_\tc(K-1, L'-1)} \leq \frac{K + H_{L'-K+1} - 2}{K - 1}, \label{eq:tc-error-bound-Psi}
\end{equation}
from which it follows directly that
\[
    f_\tc(K, L') = 1 - \frac{1 + (L'-2)f_\tc(K-1, L'-1)}{\Psi^{\tc}(K, L')} \leq 1 - \frac{K-1}{K + H_{L'-K+1} - 2} = \frac{H_{L'-K+1} - 1}{K + H_{L'-K+1} - 2}.
\]

Our strategy is to establish \eqref{eq:tc-error-bound-Psi} via strong induction over $(K, L')$, ordered by the sum $K + L'$.
Using the expressions for $ w_\ell^{\dtc} $ from \eqref{eq:wl-dtc-def} and $\Psi_\dtc(K, L)$ from \eqref{eq:Psi-def}, one can write
\begin{align*}
\Psi^{\tc}(K, L') &= \sum_{\ell = 1}^{L'-K+1} \prod_{i=1}^{\ell - 1} \frac{(L'-i)f_\tc(K-1, L' - i)}{1 + (L'-i-2)f_\tc(K-1, L' - i - 1)}.
\end{align*}

\paragraph{Base case 1: $K = 2$.}
We begin by verifying the claim when the number of batches is at its smallest, namely $K = 2$.

Setting $K = 2$, note that $f_\tc(K-1, \ell) = f_\tc(1, \ell) = 1$ for every $\ell \geq 2$. Each factor in the product that defines $\Psi^{\tc}(2, L')$ then reduces to
\[
    \frac{(L'-i)f_\tc(1, L' - i)}{1 + (L'-i-2)f_\tc(1, L' - i - 1)} = \frac{L'-i}{L'-i-1},
\]
so the product collapses telescopically to
\[
    \prod_{i=1}^{\ell - 1} \frac{L'-i}{L'-i-1} = \frac{(L'-1)(L'-2)\cdots(L'-\ell+1)}{(L'-2)(L'-3)\cdots(L'-\ell)} = \frac{L'-1}{L'-\ell}.
\]
Consequently,
\begin{align*}
    \frac{\Psi^{\tc}(2, L')}{1 + (L'-2) f_\tc(1, L'-1)} &= \frac{1}{1 + (L'-2) \cdot 1} \sum_{\ell = 1}^{L'-1} \frac{L'-1}{L'-\ell} = \frac{1}{L'-1} \sum_{\ell = 1}^{L'-1} \frac{L'-1}{L'-\ell} = \sum_{j=1}^{L'-1} \frac{1}{j} = H_{L'-1},
\end{align*}
which verifies the claim for $K = 2$.

\paragraph{Base case 2: $L' = K \geq 2$.}
We now handle the boundary case where $L'$ equals $K$.

Observe that for $L' = K$, one has $\Psi^{\tc}(K, K) = 1$, while the induction hypothesis gives $f\big(K-1, L'-(L'-K+1)\big) = f_\tc(K-1, K-1) = 0$. It follows that
\[
    \frac{\Psi^{\tc}(K, K)}{1 + (K-2) f_\tc(K-1, K-1)} = \frac{1}{1 + (K-2) \cdot 0} = 1 = \frac{K + H_1 - 2}{K - 1} = \frac{K - 1}{K - 1}
\]
and 
\[
    f_\tc(K, K) = 1 - \frac{\Psi^{\tc}(K, K)}{1 + (K-2) f_\tc(K-1, K-1)} = 0.
\] 
This settles the diagonal base case.

\paragraph{Inductive step: $L' > K > 2$.}

Suppose \eqref{eq:tc-error-bound-Psi} is valid for every pair $(K', L'')$ satisfying $K' + L'' < K + L'$. Our goal is to verify it for $(K, L')$ with $K \geq 3$ and $L' > K$, which amounts to showing 
\begin{align}\label{eq:f-bound-inductive-goal}
    \Psi^{\tc}(K, L') \leq \big(1 + (L'-2)f_\tc(K-1, L'-1) \big)\cdot \frac{K + H_{L'-K+1} - 2}{K - 1}.
\end{align}

A key ingredient is the following recurrence linking $\Psi^{\tc}(K, L')$ to $\Psi^{\tc}(K, L'-1)$:
\begin{align}
    \Psi^{\tc}(K, L') - 1 &= \frac{(L' - 1)f_\tc(K - 1, L' - 1)}{1 + (L' - 3)f_\tc(K - 1, L' - 2)} \Psi^{\tc}(K, L'-1). \label{eq:R-Psi-tc}
\end{align}
We postpone its derivation to the end of this subsection.

Taking \eqref{eq:R-Psi-tc} as given, we carry out the induction argument.
By the inductive assumption applied to the pair $(K, L'-1)$, we obtain
\[
    \frac{\Psi^{\tc}(K, L'-1)}{1 + (L' - 3)f_\tc(K - 1, L' - 2)} \leq \frac{K + H_{L'-K} - 2}{K - 1}.
\]
Substituting this estimate into \eqref{eq:R-Psi-tc} yields the following upper bound on $\Psi^{\tc}(K, L')$:
\begin{align*}
    \Psi^{\tc}(K, L')  &\leq 1+ \frac{(L' - 1)f_\tc(K - 1, L' - 1)}{1 + (L' - 3)f_\tc(K - 1, L' - 2)} \cdot \bigl(1 + (L' - 3)f_\tc(K - 1, L' - 2)\bigr) \cdot \frac{K + H_{L'-K} - 2}{K - 1} \\
    &= 1 + \frac{(L' - 1)f_\tc(K - 1, L' - 1)(K + H_{L'-K} - 2)}{K - 1}.
\end{align*}
Therefore, proving \eqref{eq:f-bound-inductive-goal} reduces to verifying that
\[
    1 + \frac{(L' - 1)f_\tc(K - 1, L' - 1)(K + H_{L'-K} - 2)}{K - 1} \leq \frac{(1 + (L'-2)f_\tc(K - 1, L' - 1))(K + H_{L'-K+1} - 2)}{K - 1}.
\]
After multiplying through by $(K - 1)$ and reorganizing, this becomes
\begin{equation}\label{eq:f-bound-key}
    (K - 1) + (L' - 1)f_\tc(K - 1, L' - 1)(K + H_{L'-K} - 2) \leq \big(1 + (L' - 2)f_\tc(K - 1, L' - 1)\big)(K + H_{L'-K+1} - 2).
\end{equation}

Isolating the factor $f_\tc(K - 1, L' - 1)$ and simplifying, \eqref{eq:f-bound-key} can be rewritten as
\[
    f_\tc(K - 1, L' - 1) \big[ (L'-1)(K + H_{L'-K} - 2) - (L'-2)(K + H_{L'-K+1} - 2) \big] \leq (K + H_{L'-K+1} - 2) - (K - 1) = H_{L'-K+1} - 1.
\]
The bracketed expression admits the following simplification:
\begin{align*}
    &(L'-1)(K + H_{L'-K} - 2) - (L'-2)(K + H_{L'-K+1} - 2) \notag\\
    &\qquad = (K + H_{L'-K} - 2) + (L'-2)\big(H_{L'-K} - H_{L'-K+1}\big) \notag\\
    &\qquad = (K + H_{L'-K} - 2) - \frac{L'-2}{L'-K+1} \notag\\
    &\qquad = K - 2 + H_{L'-K+1} - \frac{1}{L'-K+1} - \frac{L'-2}{L'-K+1} \notag\\
    &\qquad = K - 3 + H_{L'-K+1} - \frac{K - 2}{L' - K + 1}.
\end{align*}
with the second and third equalities following from the harmonic number identity $H_{L'-K+1} = H_{L'-K} + (L'-K+1)^{-1}$.
Combining the above, it remains to establish that
\begin{equation}\label{eq:f-bound-bracket-simplify}
    \Big(K - 3 + H_{L'-K+1} - \frac{K - 2}{L' - K + 1}\Big) f_\tc(K - 1, L' - 1) \leq H_{L'-K+1} - 1.
\end{equation}

To complete the argument, invoke the induction hypothesis for $(K-1, L'-1)$ --- valid because $(K-1) + (L'-1) < K + L'$ --- to deduce that
\[
    f_\tc(K - 1, L' - 1) \leq \frac{H_{(L'-1)-(K-1)+1} - 1}{(K-1) + H_{(L'-1)-(K-1)+1} - 2} = \frac{H_{L'-K+1} - 1}{K - 3 + H_{L'-K+1}}.
\]
Using this, the left-hand side of \eqref{eq:f-bound-bracket-simplify} is bounded above by
\begin{align*}
    \Big(K - 3 + H_{L'-K+1} - \frac{K - 2}{L' - K + 1}\Big)  f_\tc(K - 1, L' - 1) &\leq (K - 3 + H_{L'-K+1} )  f_\tc(K - 1, L' - 1) \\
    &\leq (K - 3 + H_{L'-K+1}) \cdot \frac{H_{L'-K+1} - 1}{K - 3 + H_{L'-K+1}} = H_{L'-K+1} - 1.
\end{align*}
This confirms \eqref{eq:f-bound-key}, thereby closing the inductive step and concluding the proof of Lemma \ref{lem:f-bound}.

\paragraph{Proof of the identity \eqref{eq:R-Psi-tc}.}
We now provide the derivation of the identity \eqref{eq:R-Psi-tc}. Note that
\begin{align*}
\Psi^{\tc}(K, L') - 1 &= \sum_{\ell = 1}^{L'-K+1} \prod_{i=1}^{\ell - 1} \frac{(L'-i)f_\tc(K-1, L' - i)}{1 + (L'-i-2)f_\tc(K-1, L' - i - 1)} - 1 \\ 
& =   \sum_{\ell=2}^{L'-K+1} \prod_{i=1}^{\ell - 1} \frac{(L'-i)f_\tc(K-1, L' - i)}{1 + (L'-i-2)f_\tc(K-1, L' - i - 1)},
\end{align*} 
since the $\ell = 1$ summand equals $1$ (empty product).

Proceeding further, we obtain
\begin{align*}
    \Psi^{\tc}(K, L') - 1 
    &= \frac{(L' - 1)f_\tc(K - 1, L' - 1)}{1 + (L' - 3)f_\tc(K - 1, L' - 2)} \sum_{\ell=2}^{L'-K+1} \prod_{i=2}^{\ell - 1} \frac{(L'-i)f_\tc(K-1, L' - i)}{1 + (L'-i-2)f_\tc(K-1, L' - i - 1)} \\ 
    &= \frac{(L' - 1)f_\tc(K - 1, L' - 1)}{1 + (L' - 3)f_\tc(K - 1, L' - 2)} \sum_{\ell = 1}^{L'-K} \prod_{i=1}^{\ell - 1} \frac{(L' - 1 - i)f_\tc(K-1, L' - 1 - i)}{1 + (L' - 1 - i- 2)f_\tc(K - 1, L' - 1 - i - 1)} \\
    &= \frac{(L' - 1)f_\tc(K - 1, L' - 1)}{1 + (L' - 3)f_\tc(K - 1, L' - 2)} \Psi^{\tc}(K, L'-1),
\end{align*}
Here, the first step separates out the $i = 1$ factor from the product, and the second step follows from a change of summation and product indices.
This completes the proof of \eqref{eq:R-Psi-tc}.

%% file: proof-DTC.tex

\subsection{Proof of Lemma \ref{lem:dtc-g}} \label{sec:proof-dtc-g}

We will prove Lemma \ref{lem:dtc-g} by induction on the number of iterations $K$. For any index set of currently masked tokens $\cI$ with $K \leq |\cI| \leq L$ and conditioned on already generated tokens $x_{\cI^\mathsf{c}}$, we aim to show that
\begin{align*}
    \E_{S \sim \pi_{\dtc}(K, \cI)} \Big[\KL\big(p_{X_{\cI}}(\cdot \mid x_{\cI^\mathsf{c}}) \,\|\, p_{\ol Y_{\cI}^S}(\cdot \mid x_{\cI^\mathsf{c}} ) \big) \Big] \leq f_\dtc(K, |\cI|)\bigg(\frac{L - |\cI|}{|\cI|} \TC \big(X_{\cI} \mid x_{\cI^\mathsf{c}}\big) + \frac{L}{|\cI|}\DTC \big(X_{\cI}\mid x_{\cI^\mathsf{c}}\big) \bigg).
\end{align*}
To simplify the notation, we denote $ L' \defn |\cI| $ throughout the proof.
Note that all distributions below are conditioned on $x_{\cI^\mathsf{c}}$.

\subsubsection{Base case} When $ K = 1 $, $ \pi_{\dtc}(1, \cI) $ simply unmasks all tokens in $\cI$ in a single iteration. 
So the sampling distribution is the product of conditional marginal distributions. By Lemma \ref{lem:single_batch_error}, we have
\begin{align*}
    &\E_{S \sim \pi_{\dtc}(1, \cI)} \Big[\KL \big(p_{X_{\cI}}(\cdot \mid x_{\cI^\mathsf{c}}) \,\|\, p_{\ol Y_{\cI}^S}(\cdot \mid x_{\cI^\mathsf{c}})  \big) \Big] \\
    & \qquad = \KL\Big(p_{X_{\cI}}(\cdot \mid x_{\cI^\mathsf{c}}) \,\|\, \prod_{i \in \cI} p_{X_i}(\cdot \mid x_{\cI^\mathsf{c}})   \Big) \\
    & \qquad = \TC \big(X_{\cI} \mid x_{\cI^\mathsf{c}} \big) \\
    & \qquad = \frac{(L'-1)(L-L')}{(L-L'+1)L'} \TC\big(X_{\cI} \mid x_{\cI^\mathsf{c}} \big) + \frac{L}{(L-L'+1)L'} \TC\big(X_{\cI} \mid x_{\cI^\mathsf{c}} \big) \\
    & \qquad  \leq \frac{(L'-1)(L-L')}{(L-L'+1)L'} \TC\big(X_{\cI} \mid x_{\cI^\mathsf{c}} \big) + \frac{(L'-1)L}{(L-L'+1)L'} \DTC\big(X_{\cI} \mid x_{\cI^\mathsf{c}} \big) \\
    & \qquad = f_\dtc(1, L')\bigg(\frac{L - L'}{L'} \TC\big(X_{\cI} \mid x_{\cI^\mathsf{c}} \big) + \frac{L}{L'}\DTC\big(X_{\cI} \mid x_{\cI^\mathsf{c}} \big) \bigg),
\end{align*}
where the inequality holds since $  \TC\big(X_{\cI} \mid x_{\cI^\mathsf{c}} \big) \leq (L'-1)\DTC\big(X_{\cI} \mid x_{\cI^\mathsf{c}} \big) $ by Lemma~\ref{lem:dtc-tc-relation} in Appendix \ref{sec:proof-dtc-tc-relation}, and the last equality holds since the definition $ f_\dtc(1, L') = \frac{L'-1}{L-L'+1} $ for any $ 1 \leq L' \leq L $.

\subsubsection{Inductive step} Assume that the claim holds for $ K - 1$. We will show that it also holds for $ K $.

\paragraph{Step 1: Bound KL error by a weighted sum $\sum_{\ell} w_\ell^\dtc b_\ell$.}
By Lemma \ref{prop:recursive}, conditioned on $x_{\cI^\mathsf{c}}$, the expected KL error of $ \pi_{\dtc}(K, \cI) $ admits the following decomposition:
\begin{align*}
    & \bE_{S \sim \pi_{\dtc}(K, \cI)} \Big[\KL\big(p_{X_{\cI}}(\cdot \mid x_{\cI^\mathsf{c}}) \,\|\, p_{\ol Y_{\cI}^S}(\cdot \mid x_{\cI^\mathsf{c}}) \big)\Big]\\ 
    & \quad = \bE_{S^{(1)}} \big[ \TC(X_{S^{(1)}}\mid x_{\cI^\mathsf{c}}) \big] + \bE_{S^{(1)},\,x_{S^{(1)}}} \bigg[ \bE_{S^{(2:K)}} \Big[\KL \big(p_{X_{\cI \setminus S^{(1)}}}(\cdot \mid x_{S^{(1)} \cup \cI^\mathsf{c}}) \,\|\, p_{\ol Y_{\cI \setminus S^{(1)}}^{S^{(2:K)}}}(\cdot \mid x_{S^{(1)} \cup \cI^\mathsf{c}} ) \big) \Big] \bigg]. 
\end{align*}
where $S^{(1)} \sim \pi_\dtc^{(1)}(K, \cI)$, $x_{S^{(1)}} \sim p_{X_{S^{(1)}}}(\cdot \mid x_{\cI^\mathsf{c}})$ and $S^{(2:K)} \sim \pi_{\dtc}(K-1, \cI \setminus S^{(1)})$. Applying the induction hypothesis to the inner expectation yields
\begin{align}
& \bE_{S} \Big[\KL\big(p_{X_{\cI}}(\cdot \mid x_{\cI^\mathsf{c}}) 	,\|\, p_{\ol Y_{\cI}^S}(\cdot 	\mid x_	{\cI^\mathsf{c}}) 	\big)\Big] 	\notag \\
    & \quad \leq \E_{S^{(1)}} \big[ \TC \big(X_{S^{(1)}}\mid x_{\cI^\mathsf{c}} \big) \big] + \E_{S^{(1)}, x_{S^{(1)}}} \bigg[ f_\dtc(K - 1, L' - |S^{(1)}|) \bigg( \frac{L-L'+|S^{(1)}| }{L'-|S^{(1)} |}\TC \big( X_{\cI \setminus S^{(1)}} \mid x_{S^{(1)}}, x_{\cI^\mathsf{c}} \big) \notag\\ 
    & \hspace{28em}+ \frac{L}{L'-|S^{(1)}| } \DTC \big( X_{\cI \setminus S^{(1)}} \mid x_{S^{(1)}}, x_{\cI^\mathsf{c}}\big) \bigg) \bigg]\notag \\
    & \quad = \E_{S^{(1)}} \bigg[ \TC \big(X_{S^{(1)}}\mid x_{\cI^\mathsf{c}} \big) + f_\dtc(K - 1, L' - |S^{(1)}|)\bigg( \frac{L-L'+|S^{(1)}| }{L'-|S^{(1)} |}\TC \big(X_{\cI \setminus S^{(1)} } \mid X_{S^{(1)} } , x_{\cI^\mathsf{c}} \big) \notag\\ 
    & \hspace{23em}+ \frac{L}{L'-|S^{(1)}| } \DTC \big(X_{\cI \setminus S^{(1)}} \mid X_{S^{(1)}}, x_{\cI^\mathsf{c}} \big) ) 	\Big] ,\label{eq:KL-decomp-DTC}
\end{align}
where the last equality follows from the definitions of conditional total correlation and conditional dual total correlation.

To analyze the right-hand side of \eqref{eq:KL-decomp-DTC}, recall that we generate the first unmasking set $S^{(1)}$ in two steps. First, we sample its size $\ell =  |S^{(1)}|$ with probability $w^\dtc_{\ell}(K,L') / \Psi^\dtc(K,L')$ for $\ell = 1, 2, \ldots, L'-K+1$. Given $\ell$, we then choose $S^{(1)}$ uniformly at random among all size-$ \ell $ subsets of $ \cI $. 
As a result, let us define $ b_\ell = b_\ell (K,L') $ as the uniform average over all size-$\ell$ subsets for each $1 \leq \ell \leq L'-K+1$:
\begin{align}
    b_{\ell}(K,L') \defn\frac{1}{\binom{L'}{\ell}}\sum_{S^{(1)} \subseteq \cI \,:\, |S^{(1)}|=\ell} \Big(&  \TC(X_{S^{(1)}} \mid x_{\cI^\mathsf{c}}) + \frac{L-L'+\ell}{L'-\ell}f_\dtc(K-1, L'-\ell)\TC(X_{\cI \setminus S^{(1)}} \mid X_{S^{(1)}}, x_{\cI^\mathsf{c}}) \notag \\
    & + \frac{L}{L'-\ell}f_\dtc(K-1, L'-\ell)\DTC(X_{\cI \setminus S^{(1)}} \mid X_{S^{(1)}}, x_{\cI^\mathsf{c}}) \Big). \label{eq:bl-DTC}
\end{align}
We can then express the expected KL sampling error on the right-hand side of \eqref{eq:KL-decomp-DTC} as a weighted sum
\begin{align}\label{eq:KL-decomp-DTC-2}
    \bE_{S \sim \pi_{\dtc}(K, \cI)} \Big[\KL\big(p_{X_{\cI}}(\cdot \mid x_{\cI^\mathsf{c}}) \,\|\, p_{\ol Y_{\cI}^S}(\cdot \mid x_{\cI^\mathsf{c}}) \big)\Big] &  \leq \sum_{\ell=1}^{L'-K+1} \bP\{|S^{(1)}|=\ell\} \cdot b_{\ell} = \frac{1}{\Psi^\dtc} \sum_{\ell=1}^{L'-K+1} w^\dtc_{\ell} b_{\ell}.
\end{align}
Therefore, this suggests we control the weighted sum $ \sum_{\ell=1}^{L'-K+1} w^\dtc_{\ell} b_{\ell} $.

\paragraph{Step 2: Express $\sum_{\ell} w^\dtc_{\ell} b_{\ell}$ as a telescoped sum.}
To this end, plugging in the definitions of conditional TC and DTC, we can rewrite $b_\ell$ defined in \eqref{eq:bl-DTC} as follows:
\begin{align*}
    b_{\ell} &= \frac{1}{\binom{L'}{\ell}}\sum_{\cS \subseteq \cI \,:\, |\cS|=\ell} \Bigg( \sum_{i \in \cS} H(X_i \mid x_{\cI^\mathsf{c}}) - H(X_{\cS} \mid x_{\cI^\mathsf{c}}) \notag \\
    &\quad + \frac{L-L'+\ell}{L'-\ell}f_\dtc(K-1, L'-\ell) \bigg( \sum_{j \in \cI \setminus \cS} H(X_j \mid X_{\cS}, x_{\cI^\mathsf{c}}) - H(X_{\cI \setminus \cS} \mid X_{\cS}, x_{\cI^\mathsf{c}}) \bigg) \notag \\
    &\quad + \frac{L}{L'-\ell}f_\dtc(K-1, L'-\ell) \bigg(H(X_{\cI \setminus \cS} \mid X_{\cS}, x_{\cI^\mathsf{c}}) - \sum_{j \in \cI \setminus \cS } H\big(X_j \mid X_{\cI \setminus \{j\}}, x_{\cI^\mathsf{c}}\big) \bigg) \Bigg)\notag \\
    &= \frac{1}{\binom{L'}{\ell}}\sum_{\cS \subseteq \cI \,:\, |\cS|=\ell} \Bigg( \sum_{i \in \cS} H(X_i \mid x_{\cI^\mathsf{c}}) - H(X_{\cS} \mid x_{\cI^\mathsf{c}}) \notag \\
    &\quad + \frac{L-L'+\ell}{L'-\ell}f_\dtc(K-1, L'-\ell) \sum_{j \in \cI \setminus \cS} H(X_j \mid X_{\cS}, x_{\cI^\mathsf{c}}) + f_\dtc(K-1, L'-\ell) H(X_{\cI \setminus \cS } \mid X_{\cS }, x_{\cI^\mathsf{c}}) \notag \\
    &\quad - \frac{L}{L'-\ell}f_\dtc(K-1, L'-\ell) \sum_{j \in \cI \setminus \cS } H\big(X_j \mid X_{\cI \setminus \{j\}}, x_{\cI^\mathsf{c}}\big)  \Bigg).
\end{align*}
As $H(X|Y) = H(X,Y) - H(Y)$ for any random variables $X,Y$, we can further rewrite $b_\ell$ as
\begin{align}
b_{\ell} &= \frac{1}{\binom{L'}{\ell}}\sum_{\cS \subseteq \cI \,:\, |\cS|=\ell} \Bigg( \sum_{i \in \cS} H(X_i \mid x_{\cI^\mathsf{c}}) - \big[1+ (L-L'+\ell+1)f_\dtc(K-1, L'-\ell) \big] H(X_{\cS} \mid x_{\cI^\mathsf{c}}) \notag \\
    &\quad + \frac{L-L'+\ell}{L'-\ell} f_\dtc(K-1, L'-\ell) \sum_{j \in \cI \setminus \cS} H\big(X_{\{j\} \cup \cS} \mid x_{\cI^\mathsf{c}}\big) + f_\dtc(K-1, L'-\ell)H(X_{\cI} \mid x_{\cI^\mathsf{c}}) \notag \\
    &\quad - \frac{L}{L'-\ell}f_\dtc(K-1, L'-\ell) \sum_{j \in \cI \setminus \cS } H\big(X_j \mid X_{\cI \setminus \{j\}}, x_{\cI^\mathsf{c}}\big)  \Bigg), \label{eq:bl-DTC-1}
\end{align}
where we use the fact that $|\cI\setminus \cS| = L'-\ell$.
Now, notice that 
\begin{align}
    \frac{1}{\ell}\frac{1}{\binom{L'}{\ell}}\sum_{\cS \subseteq \cI \,:\, |\cS|=\ell} \sum_{i \in \cS} H(X_i \mid x_{\cI^\mathsf{c}}) & = \frac{1}{L'} \sum_{i \in \cI} H(X_i \mid x_{\cI^\mathsf{c}}); \label{eq:avg-Hi-DTC}\\
\sum_{\cS \subseteq \cI \,:\,|\cS|=\ell}\sum_{j \in \cI \setminus \cS} H\big(X_{\{j\} \cup \cS} \mid x_{\cI^\mathsf{c}}\big) &= (\ell + 1) \sum_{\cT \subseteq \cI \,:\,|\cT|=\ell+1} H(X_{\cT} \mid x_{\cI^\mathsf{c}}); \label{eq:avg-Hij-DTC} \\
\frac{1}{L'-\ell}\frac{1}{\binom{L'}{\ell}} \sum_{\cS \subseteq \cI \,:\, |\cS|=\ell} \sum_{j \in \cI \setminus \cS } H\big(X_j \mid X_{\cI \setminus \{j\}}, x_{\cI^\mathsf{c}}\big) & = \frac1{L'}\sum_{j\in\cI } H\big(X_j \mid X_{\cI \setminus \{j\}}, x_{\cI^\mathsf{c}}\big). \label{eq:avg-Hj-DTC} 
\end{align}
Here, \eqref{eq:avg-Hi-DTC} holds because uniformly selecting an element from $\cI$ is equivalent to uniformly selecting a set $\cS \subseteq \cI$ with size $\ell$ and then selecting an element $i \in \cS$. \eqref{eq:avg-Hij-DTC} is true because specifying a pair $(\cS,j)$ is equivalent to removing one element from a set with size $\ell+1$, which can be done in $\ell + 1$ ways. \eqref{eq:avg-Hj-DTC} holds because uniformly selecting an element from $\cI$ is equivalent to uniformly selecting a set $\cS \subseteq \cI$ with size $\ell$ and then selecting an element $j \in \cI \setminus \cS$.

Substituting \eqref{eq:avg-Hi-DTC}--\eqref{eq:avg-Hj-DTC} into \eqref{eq:bl-DTC-1}, we can continue to rewrite $b_\ell$ as
\begin{align}
    b_{\ell}&= \frac{\ell}{L'} \sum_{i \in \cI} H(X_i \mid x_{\cI^\mathsf{c}}) - \frac{1 + (L-L'+\ell+1) f_\dtc(K-1, L'-\ell)}{\binom{L'}{\ell}} \sum_{\cS \subseteq \cI \,:\, |\cS|=\ell} H(X_{\cS} \mid x_{\cI^\mathsf{c}}) \notag\\
    &\quad + \frac{(\ell + 1)}{\binom{L'}{\ell}}\frac{L-L'+\ell}{L'-\ell} f_\dtc(K-1, L'-\ell) \sum_{\cT \subseteq \cI \,:\, |\cT|=\ell+1} H(X_{\cT} \mid x_{\cI^\mathsf{c}}) + f_\dtc(K-1, L'-\ell) H(X_{\cI} \mid x_{\cI^\mathsf{c}}) \notag \\
    & \quad - \frac{L}{L'}f_\dtc(K-1, L'-\ell) \sum_{j \in \cI} H\big(X_j \mid X_{\cI \setminus \{j\}}, x_{\cI^\mathsf{c}}\big)  \notag\\
    &= \frac{\ell}{L'} \sum_{i \in \cI} H(X_i \mid x_{\cI^\mathsf{c}}) + f_\dtc(K-1, L'-\ell) H(X_{\cI} \mid x_{\cI^\mathsf{c}}) - \frac{L}{L'}f_\dtc(K-1, L'-\ell) \sum_{j \in \cI} H\big(X_j \mid X_{\cI \setminus \{j\}}, x_{\cI^\mathsf{c}}\big)  \notag \\
    &\quad - \underbrace{\frac{1 + (L-L'+\ell+1) f_\dtc(K-1, L'-\ell)}{\binom{L'}{\ell}}}_{\defnrev\zeta_\ell}\sum_{\cS \subseteq \cI \,:\, |\cS|=\ell} H(X_{\cS} \mid x_{\cI^\mathsf{c}}) \notag\\
    &\quad + \underbrace{\frac{(L-L'+\ell) f_\dtc(K-1, L'-\ell)}{\binom{L'}{\ell + 1}}}_{\defnrev\eta_\ell} \sum_{\cT \subseteq \cI \,:\, |\cT|=\ell+1} H(X_{\cT} \mid x_{\cI^\mathsf{c}}).  \label{eq:bl-DTC-expression}
\end{align}
where the last holds because
\begin{align*}
\frac{\ell + 1}{\binom{L'}{\ell}}\frac{1}{L'-\ell} = \frac{1}{\binom{L'}{\ell + 1}}.
\end{align*}
With \eqref{eq:bl-DTC-expression} in hand, we can express the weighted sum $ \sum_{\ell} w^\dtc_{\ell} b_{\ell} $ in the expected KL error bound in \eqref{eq:KL-decomp-DTC-2} as
\begin{align}
\sum_{\ell = 1}^{L'-K+1} w^\dtc_{\ell} b_{\ell}
& = \bigg(\frac{1}{L'} \sum_{\ell = 1}^{L'-K+1}  w^\dtc_{\ell} \ell\bigg) \sum_{i \in \cI} H(X_i \mid x_{\cI^\mathsf{c}})  + \bigg(\sum_{\ell = 1}^{L'-K+1} w^\dtc_{\ell} f_\dtc(K-1, L'-\ell) \bigg) H(X_{\cI} \mid x_{\cI^\mathsf{c}})  \notag \\
& \quad - \bigg( \frac{L}{L'}\sum_{\ell = 1}^{L'-K+1} w^\dtc_{\ell} f_\dtc(K-1, L'-\ell) \bigg) \sum_{j \in \cI} H\big(X_j \mid X_{\cI \setminus \{j\}}, x_{\cI^\mathsf{c}}\big)   \notag \\
& \quad  - \sum_{\ell = 1}^{L'-K+1} \bigg(w^\dtc_{\ell} \zeta_{\ell} \sum_{\cS \subseteq \cI \,:\, |\cS|=\ell} H(X_{\cS} \mid x_{\cI^\mathsf{c}}) \bigg) + \sum_{\ell = 1}^{L'-K+1} \bigg(w^\dtc_{\ell} \eta_{\ell} \sum_{\cT \subseteq \cI \,:\, |\cT|=\ell+1} H(X_{\cT} \mid x_{\cI^\mathsf{c}}) \bigg) \notag \\ 
& = \bigg(\frac{1}{L'} \sum_{\ell = 1}^{L'-K+1}  w^\dtc_{\ell} \ell\bigg) \sum_{i \in \cI} H(X_i \mid x_{\cI^\mathsf{c}})  + \bigg(\sum_{\ell = 1}^{L'-K+1} w^\dtc_{\ell} f_\dtc(K-1, L'-\ell) \bigg) H(X_{\cI} \mid x_{\cI^\mathsf{c}}) \notag \\
& \quad  - \bigg(\frac{L}{L'} \sum_{\ell = 1}^{L'-K+1} w^\dtc_{\ell} f_\dtc(K-1, L'-\ell) \bigg) \sum_{j \in \cI} H\big(X_j \mid X_{\cI \setminus \{j\}}, x_{\cI^\mathsf{c}}\big)  \notag \\
& \quad  - w^\dtc_{1} \zeta_{1} \sum_{\cS \subseteq \cI \,:\, |\cS|=1} H(X_{\cS} \mid x_{\cI^\mathsf{c}}) + w^\dtc_{L'-K+1} \eta_{L'-K+1} \sum_{\cT \subseteq \cI \,:\, |\cT|=L'-K+2} H(X_{\cT} \mid x_{\cI^\mathsf{c}}) \notag \\ 
& \quad  + \sum_{\ell = 1}^{L'-K} \bigg((w^\dtc_{\ell} \eta_{\ell} -w^\dtc_{\ell+1} \zeta_{\ell+1}) \sum_{\cS \subseteq \cI \,:\, |\cS|=\ell+1} H(X_{\cS} \mid x_{\cI^\mathsf{c}}) \bigg). \label{eq:weighted-sum-DTC-temp}
\end{align}
where the last step holds by rearranging the terms and separating the $\ell = 1$ and $\ell = L'-K+1$ terms from the summation.

To proceed, notice that our choice of the weight $w^\dtc_{\ell}$ (see \eqref{eq:wl-dtc-def}),
\begin{align*}
w_\ell^{\dtc} = \prod_{i = 1}^{\ell-1} \frac{(L-L'+i)\,f_\dtc(K-1, L'-i)}{1 + (L-L'+i+2)\,f_\dtc(K-1, L'-i-1)},
\end{align*}
yields the following recursive relationship:
\begin{align}
w^\dtc_{\ell + 1} = w^\dtc_{\ell} \frac{(L-L'+\ell)\,f_\dtc(K-1, L'-\ell)}{1 + (L-L'+\ell+2)\,f_\dtc(K-1, L'-\ell-1)}, \quad \forall\,\ell = 1, 2, \ldots, L'-K. \label{eq:weight-recursive-DTC-1}
\end{align}
In particular, combining this with the definitions of $\zeta_\ell$ and $\eta_\ell$ in \eqref{eq:bl-DTC-expression}, we obtain that for each $1\leq \ell \leq L'-K$,
\begin{align}
w^\dtc_{\ell} \eta_\ell = w^\dtc_{\ell + 1} \zeta_{\ell + 1}. \label{eq:weight-recursive-DTC}
\end{align}

Substituting \eqref{eq:weight-recursive-DTC} into the expression of $ \sum_{\ell = 1}^{L'-K+1} w^\dtc_{\ell} b_{\ell} $ in \eqref{eq:weighted-sum-DTC-temp}, we find that the terms $ \sum_{\cS \subseteq \cI \,:\, |\cS|=\ell+1} H(X_{\cS} \mid x_{\cI^\mathsf{c}}) $ for $1\leq \ell \leq L'-K+ 1$ telescope, yielding the following simplified expression:
\begin{align}
\sum_{\ell = 1}^{L'-K+1} w^\dtc_{\ell} b_{\ell}
& = \bigg(\frac{1}{L'} \sum_{\ell = 1}^{L'-K+1}  w^\dtc_{\ell} \ell -  \zeta_{1}\bigg) \sum_{i \in \cI} H(X_i \mid x_{\cI^\mathsf{c}}) + \bigg(\sum_{\ell = 1}^{L'-K} w^\dtc_{\ell} f_\dtc(K-1, L'-\ell) \bigg) H(X_{\cI} \mid x_{\cI^\mathsf{c}}) \notag \\
& \quad  - \bigg(\sum_{\ell = 1}^{L'-K} w^\dtc_{\ell} f_\dtc(K-1, L'-\ell) \bigg) \frac{L}{L'} \sum_{j \in \cI} H\big(X_j \mid X_{\cI \setminus \{j\}}, x_{\cI^\mathsf{c}}\big) , \label{eq:weighted-sum-DTC-2}
\end{align}
where we use the fact that (1) $w^\dtc_{1}$; and (2) $g\big(K-1, L'-(L'-K+1) \big) = f_\dtc(K-1, K-1) = 0$ and hence $\eta_{L'-K+1} \propto f_\dtc(K-1, K-1) = 0$.

\paragraph{Step 3: Compute entropy coefficients.}

With the expression \eqref{eq:weighted-sum-DTC-2} in hand, it remains to compute the entropy coefficients in the above expression.

\begin{itemize}
    \item 

Let us begin with the coefficient for $H(X_{\cI} \mid x_{\cI^\mathsf{c}})$, i.e., $\sum_{\ell = 1}^{L'-K} w^\dtc_{\ell} f_\dtc(K-1, L'-\ell)$. First, we note the following algebraic manipulation of $f_\dtc(K-1, L'-\ell)$:
\begin{align*}
f_\dtc(K-1, L'-\ell) & =
- \frac{(L-L'+\ell)f_\dtc(K-1, L'-\ell)}{1 + (L-L'+\ell+2)f_\dtc(K-1, L'-\ell-1)} \Big(1 + (L-L'+\ell+2)f_\dtc(K-1, L'-\ell-1) \Big) \\ 
& \quad + (L-L'+\ell+1)f_\dtc(K-1, L'-\ell).
\end{align*}
Using this, we can compute the coefficient as follows:
\begin{align}
&\sum_{\ell = 1}^{L'-K} w^\dtc_{\ell} f_\dtc(K-1, L'-\ell) \notag\\
&\quad = \sum_{\ell = 1}^{L'-K} w^\dtc_{\ell}\biggl( - \frac{(L-L'+\ell)f_\dtc(K-1, L'-\ell)}{1 + (L-L'+\ell+2)f_\dtc(K-1, L'-\ell-1)} + (L-L'+\ell+1)f_\dtc(K-1, L'-\ell) \notag\\
&\quad \qquad \qquad\qquad - \frac{(L-L'+\ell)f_\dtc(K-1, L'-\ell)}{1 + (L-L'+\ell+2)f_\dtc(K-1, L'-\ell-1)}(L-L'+\ell+2)f_\dtc(K-1, L'-\ell-1)\biggr) \notag\\
&\quad \numpf{i}{=} -\sum_{\ell = 1}^{L'-K} w^\dtc_{\ell + 1} + \sum_{\ell = 1}^{L'-K} \Big( w^\dtc_{\ell} (L-L'+\ell+1)f_\dtc(K-1, L'-\ell) \notag\\
& \hspace{12em}- w^\dtc_{\ell + 1} \big(L-L'+(\ell + 1)+1\big)g\big(K-1, L'-(\ell + 1)\big) \Big) \notag\\
&\quad = -\sum_{\ell = 1}^{L'-K} w^\dtc_{\ell + 1} + \Big(w^\dtc_1 (L-L'+2)f_\dtc(K-1, L'-1) - w^\dtc_{L'-K+1} (L-K+2)f_\dtc(K-1, K-1) \Big) \notag\\
& \quad \numpf{ii}{=} -\sum_{\ell = 1}^{L'-K+1} w^\dtc_{\ell} + 1 + (L-L'+2)f_\dtc(K-1,L'-1) \\
& \quad = -\Psi^\dtc(K,L') + 1 + (L-L'+2)f_\dtc(K-1,L'-1) \label{eq:coeff-DTC-all}
\end{align}
Here, (i) plugs in the definition of $w^\dtc_{\ell + 1}$ in \eqref{eq:wl-dtc-def} and rearranges the terms to form a telescoping sum; (ii) holds as $w^\dtc_1 = 1$ and $f_\dtc(K-1,K-1)=0$; and the last step arises from the definition of $\Psi^\dtc = \Psi^\dtc(K,L') \defn \sum_{\ell = 1}^{L'-K+1} w^\dtc_{\ell}$ in \eqref{eq:Psi-def}.

\item The coefficient of $-\sum_{j \in \cI} H\big(X_j \mid X_{\cI \setminus \{j\}}, x_{\cI^\mathsf{c}}\big) $ is then an immediate consequence of \eqref{eq:coeff-DTC-all}:
\begin{align}
\frac{L}{L'} \sum_{\ell = 1}^{L'-K} w^\dtc_{\ell}  f_\dtc(K-1, L'-\ell) &= \frac{L}{L'} \Bigl(-\Psi^\dtc + 1 + (L-L'+2)f_\dtc(K-1,L'-1)\Bigr).\label{eq:coeff-DTC-dual}
\end{align}

\item Finally, let us turn to the coefficient of $ \sum_{i \in \cI} H(X_i \mid x_{\cI^\mathsf{c}}) $ in \eqref{eq:weighted-sum-DTC-2}. Recalling the definition of $\zeta_1$ in \eqref{eq:bl-DTC-expression}, we can compute the coefficient as follows:
\begin{align*}
&\frac1{L'} \sum_{\ell = 1}^{L'-K+1}  w^\dtc_{\ell} \ell - \zeta_{1} \notag\\
&\quad = \frac{1}{L'} \sum_{\ell = 1}^{L'-K+1} w^\dtc_{\ell} \ell  - \frac{1+(L-L'+2)f_\dtc(K-1, L'-1)}{L'} \notag \\
&\quad = -\frac{L-L'}{L'} \sum_{\ell = 1}^{L'-K+1} w^\dtc_{\ell} + \frac{1}{L'} \sum_{\ell = 1}^{L'-K+1} w^\dtc_{\ell} (L-L'+\ell) - \frac{1+(L-L'+2)f_\dtc(K-1, L'-1)}{L'} \\
&\quad = -\frac{L-L'}{L'} \Psi^\dtc + \frac{1}{L'} \sum_{\ell = 1}^{L'-K+1} w^\dtc_{\ell} (L-L'+\ell) - \frac{1+(L-L'+2)f_\dtc(K-1, L'-1)}{L'}.
\end{align*}
Again, using the algebraic manipulation
\begin{align*}
(L-L'+\ell) &= (L-L'+\ell)\big(1 + (L-L'+\ell+1)f_\dtc(K-1, L'-\ell)\big) \\
& \quad - (L-L'+\ell)(L-L'+\ell+1)f_\dtc(K-1, L'-\ell),
\end{align*}
we can further rewrite the above expression as a telescoped sum:
\begin{align}
    &\frac1{L'} \sum_{\ell = 1}^{L'-K+1}  w^\dtc_{\ell} \ell -  \zeta_{1} \notag\\
&\quad = -\frac{L-L'}{L'} \Psi^\dtc +\frac{1}{L'}w^\dtc_{L'-K+1} (L-K+1) - \frac{1+(L-L'+2)f_\dtc(K-1, L'-1)}{L'} \notag\\
& \quad\quad + \frac{1}{L'} \sum_{\ell = 1}^{L'-K} \Big( w^\dtc_{\ell} (L-L'+\ell)\big(1 + (L-L'+\ell+1)f_\dtc(K-1, L'-\ell)\big) \notag \\
&\hspace{8em} - w^\dtc_{\ell} (L-L'+\ell)(L-L'+\ell+1)f_\dtc(K-1, L'-\ell) \Big) \notag \\
&\quad \stackrel{\text{(i)}}{=} -\frac{L-L'}{L'} \Psi^\dtc +\frac{1}{L'}w^\dtc_{L'-K+1} (L-K+1) - \frac{1+(L-L'+2)f_\dtc(K-1, L'-1)}{L'} \notag\\
& \quad\quad + \frac{1}{L'} \sum_{\ell = 1}^{L'-K} \Big( w^\dtc_{\ell} (L-L'+\ell)\big(1 + (L-L'+\ell+1)f_\dtc(K-1, L'-\ell)\big) \notag \\
&\hspace{8em} - w^\dtc_{\ell + 1} \big(L-L'+(\ell + 1)\big)\big[1 + \big(L-L'+(\ell + 1)+1\big)g\big(K-1, L'-(\ell + 1)\big)\big] \Big) \notag \\
&\quad \stackrel{\text{(ii)}}{=} -\frac{L-L'}{L'} \Psi^\dtc +\frac{1}{L'}w^\dtc_{L'-K+1} (L-K+1) - \frac{1+(L-L'+2)f_\dtc(K-1, L'-1)}{L'} \notag\\
& \quad \quad + \frac{1}{L'} \Big( w^\dtc_1 (L-L'+1)\big(1 + (L-L'+2)f_\dtc(K-1, L'-1)\big) \notag \\
&\hspace{5em} - w^\dtc_{L'-K+1} (L-K+1)\big(1 + (L-K+2)f_\dtc(K-1, K-1)\big) \Big) \notag \\
&\quad \stackrel{\text{(iii)}}{=} -\frac{L-L'}{L'} \Psi^\dtc + \frac{L-L'+1}{L'} \big(1 + (L-L'+2)f_\dtc(K-1, L'-1)\big) \notag\\ 
& \quad \quad- \frac{1 + (L-L'+2)f_\dtc(K-1, L'-1)}{L'}\notag \\
&\quad = \frac{L-L'}{L'}\Bigl(-\Psi^\dtc + 1 + (L-L'+2)f_\dtc(K-1,L'-1)\Bigr). \label{eq:coeff-DTC-single}
\end{align}
Here, (i) arises from \eqref{eq:weight-recursive-DTC-1} that
$$ w^\dtc_{\ell} (L-L'+\ell)f_\dtc(K-1, L'-\ell) = w^\dtc_{\ell + 1} \big[1 + \big(L-L'+(\ell + 1)+1\big)g\big(K-1, L'-(\ell + 1)\big)\big]; $$
(ii) evaluates the telescoping sum; (iii) uses $w^\dtc_1 = 1$ and $f_\dtc(K-1,K-1) = 0$.

\end{itemize}

\paragraph{Step 4: Putting everything together.}
Combining \eqref{eq:coeff-DTC-all}--\eqref{eq:coeff-DTC-dual} with \eqref{eq:weighted-sum-DTC-2}, we find that
\begin{align}
\sum_{\ell = 1}^{L'-K+1} w^\dtc_{\ell} b_{\ell}
& = \Bigl(-\Psi^\dtc + 1 + (L-L'+2)f_\dtc(K-1, L'-1)\Bigr) \notag\\
& \quad \cdot\bigg(\frac{L-L'}{L'}\sum_{i \in \cI} H(X_i \mid x_{\cI^\mathsf{c}}) + H(X_{\cI} \mid x_{\cI^\mathsf{c}}) - \frac{L}{L'} \sum_{j \in \cI} H(X_j \mid X_{\cI \setminus \{j\}}, x_{\cI^\mathsf{c}})\bigg) \notag\\
& = \Bigl(-\Psi^\dtc + 1 + (L-L'+2)f_\dtc(K-1, L'-1)\Bigr) \biggl( \frac{L-L'}{L'}\TC\big(X_{\cI} \mid x_{\cI^\mathsf{c}} \big) + \frac{L}{L'}\DTC\big(X_{\cI} \mid x_{\cI^\mathsf{c}} \big)\biggr). \label{eq:weighted-sum-DTC}
\end{align}
In words, the weighted sum is a linear combination of TC and DTC of the tokens from the index set $ \cI $ conditioned on $ x_{\cI^\mathsf{c}} $: 

Plugging \eqref{eq:weighted-sum-DTC} into the expected KL error bound in \eqref{eq:KL-decomp-DTC-2}, we conclude that
\begin{align}
    &\bE_{S \sim \pi_{\dtc}(K, \cI)} \Big[\KL\big(p_{X_{\cI}}(\cdot \mid x_{\cI^\mathsf{c}}) \,\|\, p_{\ol Y_{\cI}^S}(\cdot \mid x_{\cI^\mathsf{c}}) \big)\Big] 
    \leq \frac{1}{\Psi^\dtc}\sum_{\ell = 1}^{L'-K+1} w^\dtc_{\ell} b_{\ell} \notag\\
     &\qquad = \biggl(-1 + \frac{1 + (L-L'+2)f_\dtc(K-1, L'-1)}{\Psi^\dtc(K, L')}\biggr) \bigg( \frac{L-L'}{L'}\TC\big(X_{\cI} \mid x_{\cI^\mathsf{c}} \big) + \frac{L}{L'}\DTC\big(X_{\cI} \mid x_{\cI^\mathsf{c}} \big)\bigg) \notag\\
     &\qquad = f_\dtc(K, L') \bigg( \frac{L-L'}{L'}\TC\big(X_{\cI} \mid x_{\cI^\mathsf{c}} \big) + \frac{L}{L'}\DTC\big(X_{\cI} \mid x_{\cI^\mathsf{c}} \big)\bigg), \label{eq:KL-bound-final-DTC}
\end{align}
where the last equality follows from the definition of $ f_\dtc(K, L') $ in \eqref{eq:g-def}.
This completes the inductive step and hence the proof of Lemma \ref{lem:dtc-g}.

%% file: proof-DTC-g.tex

\subsection{Proof of Lemma \ref{lem:g-bound}}
\label{sec:proof-g-bound}

To see why \eqref{eq:dtc-g-final} is an immediate consequence of \eqref{eq:dtc-g}, notice that $\min\{ L - K, \frac{L}{K}H_{L-1} \} < L$.
Since $x \mapsto x/(L-x)$ is increasing on $[0,L)$, we arrive at the desired bound:
\begin{align*}
\frac{\min\{L-K, \frac{L}{K}H_{L-1} \}}{L - \min\{L-K, \frac{L}{K}H_{L-1} \}} \leq \frac{\frac{L}{K}H_{L-1}}{L - \frac{L}{K}H_{L-1}} = \frac{ H_{L-1}}{K - H_{L-1}}.
\end{align*}

Therefore, the remainder of this section is devoted to proving \eqref{eq:dtc-g}.
In light of the definition of $ f_\dtc(K, L) $ in \eqref{eq:g-def}, it suffices to show that for any $K \geq 2$ and $K \leq L' \leq L$,
\begin{equation}
    \frac{\Psi_\dtc(K, L')}{1 + (L-L'+2)f_\dtc(K-1, L'-1)} \geq \frac{L-\min\{L'-K, \frac{L}{K}(H_{L-1} - H_{L-L'} )\}}{L} > 0, \label{eq:dtc-error-bound-psi}
\end{equation}
which immediately implies that
\begin{align*}
    f_\dtc(K, L') &= -1 + \frac{1 + (L-L'+2)f_\dtc(K-1, L'-1)}{\Psi_\dtc(K, L')} \\
    &\leq -1 + \frac{L}{L-\min\{L'-K, \frac{L}{K}(H_{L-1} - H_{L-L'} )\}}= \frac{\min\{L'-K, \frac{L}{K}(H_{L-1} - H_{L-L'} )\}}{L - \min\{L'-K, \frac{L}{K}(H_{L-1} - H_{L-L'} )\}}.
\end{align*}

Towards this, we introduce the shorthand notation
\begin{align}\label{eq:p-def}
    p(K, L') \defn \min \Big\{L'-K, \frac{L}{K}(H_{L-1} - H_{L-L'} ) \Big\}
\end{align}
for simplicity and presentation, and recall the definition of $ w_l^{\dtc} $ in \eqref{eq:wl-dtc-def} and $\Psi_\dtc(K, L)$ in \eqref{eq:Psi-def}, we have
\begin{align*}
\Psi_\dtc(K, L') &= \sum_{\ell = 1}^{L'-K+1} \prod_{i=1}^{\ell - 1} \frac{(L-L'+i)f_\dtc(K-1, L'-i)}{1 + (L-L'+i+2)f_\dtc(K-1, L'-i-1)}.
\end{align*}
In what follows, we will prove $$ \frac{\Psi_\dtc(K, L')}{1 + (L-L'+2)f_\dtc(K-1, L'-1)} \geq 1 - \frac{p(K, L')}{L} $$  by strong induction on the pair $(K, L')$, or equivalently, the sum $K + L'$.

\subsubsection{Base case 1: $K = 2$}
For $K = 2$, we have $$f_\dtc(K-1, L') = f_\dtc(1, L') = \frac{L'-1}{L-L'+1}$$ for all $L' \geq 1$. Thus, each term in the product defining $\Psi_\dtc(2, L')$ simplifies as
\[
    \frac{(L-L'+i)f_\dtc(K-1, L'-i)}{1 + (L-L'+i+2)f_\dtc(K-1, L'-i-1)} = \frac{(L-L'+i)\frac{L'-i-1}{L-L'+i+1}}{1 + (L-L'+i+2)\frac{L'-i-2}{L-L'+i+2}} = \frac{L-L'+i}{L-L'+i+1}.
\]
and the product telescopes as
\[
    \prod_{i=1}^{\ell - 1} \frac{(L-L'+i)f_\dtc(K-1, L'-i)}{1 + (L-L'+i+2)f_\dtc(K-1, L'-i-1)} = \prod_{i=1}^{\ell - 1} \frac{L-L'+i}{L-L'+i+1} = \frac{L-L'+1}{L-L'+\ell}.
\]

Combining the expressions of $\Psi_\dtc(2, L')$ and $f_\dtc(1, L'-1)$, we can bound
\begin{align*}
\frac{\Psi_\dtc(2, L')}{1 + (L-L'+2)f_\dtc(1, L'-1)} &= \frac{L-L'+1}{L'-1} \sum_{\ell = 1}^{L'-1} \frac{1}{L-L'+\ell} \\ 
& \numpf{i}{\geq} \frac{L-L'+1}{L'-1} \frac{(L'-1)^2}{\sum_{\ell = 1}^{L'-1} (L-L'+\ell)} \\
&= \frac{2(L-L'+1)}{2L - L'} = 1 - \frac{L'-2}{2L-L'} \\
&\geq 1 - \frac{L'-2}{L}, 
\end{align*}
where (i) arises from 
\begin{align}\label{eq:cs-1}
\sum_{\ell = 1}^{L'-1} \frac{1}{L-L'+\ell} \geq \frac{(L'-1)^2}{\sum_{\ell = 1}^{L'-1} (L-L'+\ell)}.
\end{align}
This is due to the Cauchy-Schwarz inequality: for any positive numbers $a_1, a_2, \ldots, a_n$, $\sum_{i=1}^n \frac{1}{a_i} \geq \frac{n^2}{\sum_{i=1}^n a_i}$.

Meanwhile, one can also bound
\begin{align*}
\frac{\Psi_\dtc(2, L')}{1 + (L-L'+2)f_\dtc(1, L'-1)} &= \frac{L-L'+1}{L'-1} \sum_{\ell = 1}^{L'-1} \frac{1}{L-L'+\ell} \\
&= \frac{L-L'+1}{L'-1} (H_{L-1} - H_{L-L'} ) \\
&= \Big(1 - \frac{H_{L-1} - H_{L-L'} }{2}\Big) + \Big(\frac{L-L'+1}{L'-1} + \frac{1}{2}\Big)(H_{L-1} - H_{L-L'} ) - 1 \\
&\geq 1 - \frac{\frac{L}{2}(H_{L-1} - H_{L-L'} )}{L} + \Big(\frac{L-L'+1}{L'-1} + \frac{1}{2}\Big)\frac{(L'-1)^2}{\sum_{\ell = 1}^{L'-1}(L-L'+\ell)} - 1 \\
&= 1 - \frac{\frac{L}{2}(H_{L-1} - H_{L-L'} )}{L} + \frac{2L-L'+1}{2L-L'} - 1 \\
&\geq 1 - \frac{\frac{L}{2}(H_{L-1} - H_{L-L'} )}{L},
\end{align*}
where the first inequality applies \eqref{eq:cs-1} again.
Combining these two bounds with the definition of $p(K, L')$ in \eqref{eq:p-def}, we conclude that
\[
    \frac{\Psi_\dtc(2, L')}{1 + (L-L'+2)f_\dtc(1, L'-1)} \geq \max \bigg\{1 - \frac{L'-2}{L}, 1 - \frac{\frac{L}{2}(H_{L-1} - H_{L-L'} )}{L}\bigg\} = 1 - \frac{p(2, L')}{L},
\] 
 confirming the base case $K = 2$.

\subsubsection{Base case 2: $L' = K > 2$}
Next, let us consider the ``diagonal boundary'' of the region $L' \geq K$.

In this case, it is not hard to see that $w_1^\dtc(K,K) =\Psi_\dtc(K, K) = 1$ and $p(K, K) = 0$ for any $K\geq 2$.
In particular, by the definition of $f_\dtc$ in \eqref{eq:g-def} and the base case $f_\dtc(1, 1) = 0$, we can use induction on $K$ to show that
$f_\dtc(K-1, K-1) = 0$. Hence, we find that
\[
    \frac{\Psi_\dtc(K, K)}{1 + (L-L'+2)f_\dtc(K-1, K-1)} 
     = 1 
    = 1 - \frac{p(K, K)}{L}.
\]
This establishes the base case $L' = K$.

\subsubsection{Inductive step: $L' > K > 2$}

Now, let us assume \eqref{eq:dtc-error-bound-psi} holds for all pairs $(K', L'')$ with $K' + L'' < K + L'$. We prove it for the pair $(K, L')$, i.e.,
\begin{align}\label{eq:g-bound-inductive-goal}
    \Psi_\dtc(K, L') \geq \big(1 + (L - L' + 2)f_\dtc(K-1, L'-1)\big) \bigg(1 - \frac{p(K, L')}{L} \bigg).
\end{align}

Towards this, we first present the following identity, which relates relates $\Psi_\dtc(K, L')$ to $\Psi_\dtc(K, L'-1)$:
\begin{align}
    \Psi_\dtc(K, L') &= 1+ \frac{(L-L'+1)f_\dtc(K - 1, L' - 1)}{1 + (L-L'+3)f_\dtc(K - 1, L' - 2)} \Psi_\dtc(K, L'-1). \label{eq:R-Psi}
\end{align}
The proof is deferred to the end of this section.

Assuming \eqref{eq:R-Psi} for the moment, we proceed with the inductive step.
Applying the induction hypothesis to $(K, L'-1)$, one knows that
\[
    \frac{\Psi_\dtc(K, L'-1)}{1 + (L-L'+3)f_\dtc(K - 1, L' - 2)} \geq 1 - \frac{p(K, L'-1)}{L}.
\]
Plugging this bound into \eqref{eq:R-Psi}, one can control $\Psi_\dtc(K, L')$ as
\begin{align*}
    \Psi_\dtc(K, L')  &\geq 1+ \frac{(L-L'+1)f_\dtc(K-1, L'-1)}{1 + (L-L'+3)f_\dtc(K-1, L'-2)} \bigl(1 + (L-L'+3)f_\dtc(K - 1, L' - 2)\bigr) \bigg(1 - \frac{p(K, L'-1)}{L}\bigg) \\
    & = 1 + (L-L'+1)f_\dtc(K-1, L'-1)\bigg(1 - \frac{p(K, L'-1)}{L}\bigg).
\end{align*}

Hence, to establish \eqref{eq:g-bound-inductive-goal}, it suffices to show that
\[
    1 + (L-L'+1)f_\dtc(K-1, L'-1)\bigg(1 - \frac{p(K, L'-1)}{L}\bigg) \geq \big(1 + (L - L' + 2)f_\dtc(K-1, L'-1)\big) \bigg(1 - \frac{p(K, L')}{L} \bigg).
\]
By straightforward algebraic manipulation, this is equivalent to showing that
\begin{align*}
    &(L-L'+2) \big(L - p(K, L')\big) - (L-L'+1) \big(L - p(K, L'-1)\big)  \leq \frac{p(K, L')}{f_\dtc(K-1, L'-1)} .
\end{align*}
Moreover, by the induction hypothesis applied to $(K-1, L'-1)$, we have $$f_\dtc(K-1, L'-1) \leq \frac{p(K-1, L'-1)}{L - p(K-1, L'-1)}.$$ Hence, it suffices to show that
\begin{align*}
    &\Big((L-L'+2)\big(L - p(K, L')\big) - (L-L'+1)(L - p(K, L'-1) \big)\Big) \frac{p(K-1, L'-1)}{L - p(K-1, L'-1)} \leq p(K, L'),
\end{align*}
or equivalently,
\begin{align}
    \big( L - (L-L'+2)\,p(K, L') + (L-L'+1)\,p(K, L'-1) \big) \frac{p(K-1, L'-1)}{L - p(K-1, L'-1)} \leq p(K, L').
    \label{eq:g-bound-key}
\end{align}

In what follows, we prove \eqref{eq:g-bound-key} by cases depending on the value of $p(K, L')$.

\paragraph{Case 1: $p(K, L') = L'-K$.}
By the definition of $p(K, L')$ in \eqref{eq:p-def}, we have $p(K, L'-1) \leq L'-K-1$ and $p(K-1, L'-1) \leq L'-K$. Hence, we can bound
\begin{align*}
    L - (L-L'+2)\,p(K,L') + (L-L'+1)\,p(K, L'-1) &\leq L - (L-L'+2)(L'-K) + (L-L'+1)(L'-K-1) \\
    & = K - 1,
\end{align*}
and $$\frac{p(K-1, L'-1)}{L-p(K-1, L'-1)} \leq \frac{L'-K}{L-L'+K}.$$
Combining these two bounds, the left-hand side of \eqref{eq:g-bound-key} is at most
\[
    \frac{(K-1)(L'-K)}{L-L'+K} \leq L'-K = p(K, L'),
\]
where the inequality holds because $L' \leq L$. Hence, \eqref{eq:g-bound-key} holds in this case.

\paragraph{Case 2: $p(K, L') = \frac{L}{K}(H_{L-1} - H_{L-L'})$.}
Since $p(K, L'-1) \leq \frac{L}{K}(H_{L-1} - H_{L-L'+1})$ by definition, we can derive
\begin{align}
    &L - (L-L'+2)\,p(K, L') + (L-L'+1)\,p(K, L'-1) \notag \\
    &\qquad \leq L - (L-L'+2)\frac{L}{K}(H_{L-1} - H_{L-L'}) + (L-L'+1)\frac{L}{K}(H_{L-1} - H_{L-L'+1}) \notag \\
    &\qquad = L - \frac{L}{K}\Big((L-L'+2)(H_{L-1} - H_{L-L'}) - (L-L'+1)\big((H_{L-1} - H_{L-L'}) - (L-L'+1)^{-1}\big)\Big) \notag \\
    &\qquad = L - \frac{L}{K}\big(H_{L-1} - H_{L-L'} + 1\big) = \frac{L}{K}\big(K - 1 - (H_{L-1} - H_{L-L'})\big). \label{eq:g-bound-key-lhs}
\end{align}
where the first equality uses the identity $H_{L-1} - H_{L-L'+1} = (H_{L-1} - H_{L-L'}) - (L-L'+1)^{-1}$.

If $K - 1 - (H_{L-1} - H_{L-L'})\leq 0$, then the left-hand side of \eqref{eq:g-bound-key} is non-positive and \eqref{eq:g-bound-key} holds trivially. Otherwise, let us consider the following two sub-cases based on the value of $p(K-1,L'-1) = \min\big\{L'-K,\; \frac{L}{K-1}(H_{L-1} - H_{L-L'+1})\big\}$.
\begin{itemize}
    \item 
\textbf{Sub-case 2a: $p(K-1, L'-1) = \frac{L}{K-1}(H_{L-1} - H_{L-L'+1})$.}
Since $p(K-1, L'-1) \leq L'-K < L$ by definition, the condition implies that $H_{L-1} - H_{L-L'+1} < K - 1$. Hence, we know that
\[
    \frac{p(K-1, L'-1)}{L - p(K-1, L'-1)} = \frac{H_{L-1} - H_{L-L'+1}}{K - 1 - (H_{L-1} - H_{L-L'+1})} > 0.
\]
Together with \eqref{eq:g-bound-key-lhs}, this yields that the left-hand side of \eqref{eq:g-bound-key} is at most
\[
    \frac{L\big(K - 1 - (H_{L-1} - H_{L-L'})\big)\,(H_{L-1} - H_{L-L'+1})}{K\big(K - 1 - (H_{L-1} - H_{L-L'+1})\big)}.
\]
This is upper bounded by $p(K, L') = \frac{L}{K}(H_{L-1} - H_{L-L'})$ if and only if
\[
    \big(K - 1 - (H_{L-1} - H_{L-L'})\big)\,(H_{L-1} - H_{L-L'+1}) \leq (H_{L-1} - H_{L-L'})\,\big(K - 1 - (H_{L-1} - H_{L-L'+1})\big),
\]
or equivalently,
\[
    (K-1)(H_{L-1} - H_{L-L'+1}) \leq (K-1)(H_{L-1} - H_{L-L'}).
\]
This holds since $H_{n}$ is increasing in $n$. Hence, \eqref{eq:g-bound-key} holds in Sub-case 2a.

\item\textbf{Sub-case 2b: $p(K-1, L'-1) = L'-K$.}
In this case, one has $$\frac{p(K-1, L'-1)}{L-p(K-1, L'-1)} = \frac{L'-K}{L-L'+K}.$$ Combining this with \eqref{eq:g-bound-key-lhs}, we can control the left-hand side of \eqref{eq:g-bound-key} by
\[
    \frac{L\big(K -1 - (H_{L-1} - H_{L-L'})\big)(L'-K)}{K(L-L'+K)}.
\]
This is upper bounded by $p(K, L') = \frac{L}{K}(H_{L-1} - H_{L-L'})$ if and only if
\[
    \big(K -1 - (H_{L-1} - H_{L-L'}) \big)(L'-K) \leq (H_{L-1} - H_{L-L'})(L-L'+K),
\]
or equivalently,
$$(K-1)(L'-K) \leq (H_{L-1} - H_{L-L'})L.$$ 
Since the condition $p(K-1, L'-1) = L'-K$ means that $\frac{L}{K-1}(H_{L-1} - H_{L-L'+1}) \geq L'-K$, we obtain
\[
    L(H_{L-1} - H_{L-L'}) = L(H_{L-1} - H_{L-L'+1}) + \frac{L}{L-L'+1} > (K-1)(L'-K),
\]
which verifies the desired inequality.
Thus \eqref{eq:g-bound-key} also holds in Sub-case 2b.

\item Combining these two sub-cases establishes \eqref{eq:g-bound-key} in Case 2.
\end{itemize}

Finally, combining Cases 1 and 2 concludes that \eqref{eq:g-bound-key} holds for $(K, L')$, thereby finishing the inductive step. The proof of Lemma~\ref{lem:g-bound} is now complete.

\paragraph{Proof of the identity \eqref{eq:R-Psi}.}
It remains to prove the identity \eqref{eq:R-Psi}.
Recall the definition of $\Psi_\dtc(K, L')$, we can write
\begin{align*}
\Psi_\dtc(K, L') - 1 &= \sum_{\ell=1}^{L'-K+1} \prod_{i=1}^{\ell-1} \frac{(L-L'+i)f_\dtc(K-1, L' - i)}{1 + (L-L'+i+2)f_\dtc(K-1, L' - i - 1)} - 1 \\ 
& =   \sum_{\ell=2}^{L'-K+1} \prod_{i=1}^{\ell-1} \frac{(L-L'+i)f_\dtc(K-1, L' - i)}{1 + (L-L'+i+2)f_\dtc(K-1, L' - i - 1)},
\end{align*} 
where the last equality follows from the fact that the $\ell = 1$ term in the sum is $1$.

We can further derive
\begin{align*}
    \Psi_\dtc(K, L') - 1 
    &= \frac{(L-L'+1)f_\dtc(K-1, L'-1)}{1 + (L-L'+3)f_\dtc(K-1, L'-2)} \sum_{\ell=2}^{L'-K+1} \prod_{i=2}^{\ell-1} \frac{(L-L'+i)f_\dtc(K-1, L' - i)}{1 + (L-L'+i+2)f_\dtc(K-1, L' - i - 1)} \\ 
    &= \frac{(L-L'+1)f_\dtc(K-1, L'-1)}{1 + (L-L'+3)f_\dtc(K-1, L'-2)} \sum_{\ell=1}^{L'-K} \prod_{i=1}^{\ell-1} \frac{(L-L'+1+i)f_\dtc(K-1, L' - 1 - i)}{1 + (L-L'+1+i+2)f_\dtc(K - 1, L' - 1 - i - 1)} \\
    &= \frac{(L-L'+1)f_\dtc(K-1, L'-1)}{1 + (L-L'+3)f_\dtc(K-1, L'-2)} \Psi_\dtc(K, L'-1).
\end{align*}
where the first line extracts the $i = 1$ factor in the product and the second line holds by re-indexing indices.
This finishes the proof of \eqref{eq:R-Psi}.

%% file: proof-aux.tex
\section{Proof of auxiliary lemmas}
\subsection{Proof of \eqref{eq:train-sample-decouple}}
\label{sec:proof-train-sample-decouple}
We follow the argument in \citep{li2025breaking} to decompose the prediction error from the sampling error. 

For any fixed unmasking sets $ S = (S^{(1)}, S^{(2)}, \ldots, S^{(K)} ) $, we have
\begin{align*}
    &\KL\big(p_{X^{(0)}}(\cdot) \,\|\, p_{Y_{\out}^S}(\cdot)\big) - \KL \big(p_{X^{(0)}}(\cdot) \,\|\, p_{\ol Y^S_{\out}}(\cdot)\big) \\
    &\quad = \int_{\mathbb{X}^L} p_{X^{(0)}}(x) \log \frac{p_{\ol Y^S_{\out}}(x)}{p_{Y_{\out}^S}(x)} \diff x \\
    &\quad \stackrel{\text{(i)}}{=} \sum_{k=1}^K \int_{\mathbb{X}^L} p_{X^{(0)}}(x) \log \frac{p^\star(x_{S^{(k)} }  \mid X_{U^{(k-1)} } = x_{U^{(k-1)} }  )}{\widehat{p}(x_{S^{(k)} }  \mid X_{U^{(k-1)} } = x_{U^{(k-1)} }  )} \diff x \\
    &\quad \stackrel{\text{(ii)}}{=} \sum_{k=1}^K \int_{\mathbb{X}^L} p_{X^{(0)}}(x) \sum_{i \in S^{(k)} } \log \frac{p_i^\star(x_i \mid X_{U^{(k-1)} } = x_{U^{(k-1)} } )}{\widehat{p}_i(x_i \mid X_{U^{(k-1)} } = x_{U^{(k-1)} } )} \diff x \\
    &\quad \stackrel{\text{(iii)}}{=} \mathbb{E}_{\tau, X^{(0)}} \Bigg[ \frac{L}{|S^{(\tau)}| } \sum_{i \in S^{(\tau)} } \log \frac{p_i^\star(X^{(0)}_i \mid X^{(\tau)} )}{\widehat{p}_i(X^{(0)}_i\mid X^{(\tau)} )} \, \Big| \, S \Bigg],
\end{align*}
where (i) follows from the chain rule and the definition of $ p_{Y_{\out}^S} $ and $ p_{\ol Y^S_{\out}} $;
(ii) uses the fact that $ p^{\star}  $ and $ \hat{p} $ are product distributions; 
(iii) holds because of the definition $ \bP\{\tau = k \mid S\} = \frac{1}{L}|S^{(\tau)} | $.
Taking the expectation over $ S \sim \pi $, we obtain \eqref{eq:train-sample-decouple}:
\begin{align*}
    &\mathbb{E}_{S \sim \pi} \Big[ \KL\big(p_{X^{(0)}}(\cdot) \,\|\, p_{Y_{\out}^S}(\cdot)\big) - \KL\big(p_{X^{(0)}}(\cdot) \,\|\, p_{\ol Y^S_{\out}}(\cdot)\big) \Big] \\
    &\quad = \mathbb{E}_{\tau, X^{(0)}, S} \Bigg[ \frac{L}{|M^{(\tau)} |} \sum_{i \in M^{(\tau)} } \log \frac{p_i^\star(X^{(0)}_i \mid X^{(\tau)} )}{\widehat{p}_i(X^{(0)}_i\mid X^{(\tau)} )} \Bigg] = \veps_{\train}(\pi).
\end{align*}

\subsection{Proof of Lemma \ref{lem:single_batch_error}}
\label{sec:proof-single_batch_error}

Note that the sampled distribution of the first batch is the product of conditional marginals:
\[
    \left(\prod_{i \in S} p_{X_i}(\cdot \mid x_{\cI^\mathsf{c}})\right)(x_S) = \prod_{i \in S} p_{X_i}(x_i \mid x_{\cI^\mathsf{c}}).
\]
Hence, the KL divergence can be expressed as
\begin{align*}
    \KL\big(p_{X_S}(\cdot \mid x_{\cI^\mathsf{c}}) \| \prod_{i \in S} p_{X_i}(\cdot \mid x_{\cI^\mathsf{c}})\big) &= \sum_{x_S} p_{X_S}(x_S \mid x_{\cI^\mathsf{c}}) \log \frac{p_{X_S}(x_S \mid x_{\cI^\mathsf{c}})}{\prod_{i \in S} p_{X_i}(x_i \mid x_{\cI^\mathsf{c}})} \\
    &= \sum_{i \in S} H(X_i \mid x_{\cI^\mathsf{c}}) - H(X_S \mid x_{\cI^\mathsf{c}}) \\
    &= \TC(X_S \mid x_{\cI^\mathsf{c}}).
\end{align*}

\subsection{Relationship between TC and DTC}
\label{sec:proof-dtc-tc-relation}
\begin{lemma}
\label{lem:dtc-tc-relation}
For any random vector $X = (X_1, X_2, \ldots, X_n)$, we have
\[
    \TC(X) \leq (n-1)\DTC(X).
\]
\end{lemma}
\begin{proof}
Consider the $n$ cyclic permutations $\sigma_0, \sigma_1, \ldots, \sigma_{n-1}$ of $[n]$, where $\sigma_j(k) = ((k + j - 1) \bmod n) + 1$. Applying the chain rule of entropy to each cyclic permutation $\sigma_j$ gives
\begin{align}
    H(X_{[n]}) = \sum_{k=1}^n H\big(X_{\sigma_j(k)} \mid X_{\sigma_j(1)}, \ldots, X_{\sigma_j(k-1)}\big). \label{eq:chain-rule-cyclic}
\end{align}
Summing \eqref{eq:chain-rule-cyclic} over all $n$ cyclic permutations yields
\begin{align}
    n \, H(X_{[n]}) = \sum_{j=0}^{n-1} \sum_{k=1}^n H\big(X_{\sigma_j(k)} \mid X_{\sigma_j(1)}, \ldots, X_{\sigma_j(k-1)}\big). \label{eq:sum-cyclic}
\end{align}
Now we examine the contribution of each $X_i$ to the right-hand side. Across the $n$ cyclic permutations, $X_i$ appears exactly once in each position $k = 1, \ldots, n$. In particular:
\begin{itemize}
    \item When $k = 1$: $X_i$ is conditioned on nothing, contributing $H(X_i)$.
    \item When $k = n$: $X_i$ is conditioned on all others, contributing $H(X_i \mid X_{[n] \setminus \{i\}})$.
    \item For $k = 2, \ldots, n-1$: $X_i$ is conditioned on some strict subset of $[n] \setminus \{i\}$.
\end{itemize}
Since conditioning reduces entropy, i.e., $H(X_i \mid X_T) \geq H(X_i \mid X_{[n] \setminus \{i\}})$ for any $T \subseteq [n] \setminus \{i\}$, the $n - 2$ middle terms are each lower bounded by $H(X_i \mid X_{[n] \setminus \{i\}})$. Therefore, for each $X_i$, its total contribution satisfies
\begin{align}
    \sum_{k=1}^n H(X_{\sigma_j(k)} \mid X_{\sigma_j(1)}, \ldots, X_{\sigma_j(k-1)}) \geq H(X_{\sigma_j(1)} ) + (n-1) H(X_{\sigma_j(1)}  \mid X_{[n] \setminus \{i\}}). \label{eq:per-variable-lb}
\end{align}
Summing \eqref{eq:per-variable-lb} over all $i \in [n]$ and substituting into \eqref{eq:sum-cyclic}, we obtain
\begin{align}
    n \, H(X_{[n]}) \geq \sum_{i=1}^n H(X_i) + (n-1) \sum_{i=1}^n H(X_i \mid X_{[n] \setminus \{i\}}). \label{eq:nH-lb}
\end{align}
Rearranging \eqref{eq:nH-lb} gives
\begin{align*}
    \sum_{i=1}^n H(X_i) - H(X_{[n]}) \leq (n-1) \bigg(H(X_{[n]}) - \sum_{i=1}^n H(X_i \mid X_{[n] \setminus \{i\}}) \bigg).
\end{align*}
This is equivalent to the desired inequality:
\[
    \TC(X) \leq (n-1)\,\DTC(X).
\]
\end{proof}